\documentclass{article}

\usepackage{microtype}
\usepackage{graphicx}
\usepackage{subfigure}
\usepackage{booktabs} 

\usepackage{hyperref}

\usepackage[accepted]{icml2023}

\usepackage{icml_headers}

\usepackage{amsmath}
\usepackage{amssymb}
\usepackage{mathtools}
\usepackage{amsthm}

\usepackage[capitalize,noabbrev]{cleveref}

\theoremstyle{plain}
\newtheorem{theorem}{Theorem}[section]

\newtheorem{lemma}[theorem]{Lemma}

\theoremstyle{definition}
\newtheorem{definition}[theorem]{Definition}
\newtheorem{assumption}[theorem]{Assumption}
\theoremstyle{remark}
\newtheorem{remark}[theorem]{Remark}

\usepackage[textsize=tiny]{todonotes}

\icmltitlerunning{Corruption-Robust Algorithms with Uncertainty Weighting for Nonlinear Contextual Bandits and MDPs}

\begin{document}

\twocolumn[
\icmltitle{Corruption-Robust Algorithms with Uncertainty Weighting for Nonlinear Contextual Bandits and Markov Decision Processes}




\begin{icmlauthorlist}
\icmlauthor{Chenlu Ye}{xxx}
\icmlauthor{Wei Xiong}{xxx}
\icmlauthor{Quanquan Gu}{zzz}
\icmlauthor{Tong Zhang}{xxx}
\end{icmlauthorlist}

\icmlaffiliation{xxx}{The Hong Kong University of Science and Technology;}
\icmlaffiliation{zzz}{Department of Computer Science, University of California, Los Angeles;}

\icmlcorrespondingauthor{Tong Zhang}{tongzhang@tongzhang-ml.org}


\icmlkeywords{Machine Learning, ICML}

\vskip 0.3in
]




\printAffiliationsAndNotice{}

\begin{abstract}
Despite the significant interest and progress in reinforcement learning (RL) problems with adversarial corruption, current works are either confined to the linear setting or lead to an undesired $\tilde{\cO}(\sqrt{T}\zeta)$ regret bound, where $T$ is the number of rounds and $\zeta$ is the total amount of corruption. In this paper, we consider contextual bandits with general function approximation and propose a computationally efficient algorithm to achieve a regret of $\tilde{\cO}(\sqrt{T}+\zeta)$. The proposed algorithm relies on the recently developed uncertainty-weighted least-squares regression from linear contextual bandits \citep{he2022nearly} and a new weighted estimator of uncertainty for the general function class. In contrast to the existing analysis for the sum of uncertainty that is heavily based on the linear structure, we develop a novel technique to control the sum of weighted uncertainty, thus establishing the final regret bound. We then generalize our algorithm to the episodic MDP and first achieve an additive dependence on the corruption level $\zeta$ in the scenario of general function approximation. Notably, our algorithms achieve regret bounds that either nearly match the lower bound or improve the performance of existing methods for all the corruption levels in both known and unknown $\zeta$ cases.
\end{abstract}

\section{Introduction}
This paper studies contextual bandits \citep{langford2007epoch} and episodic Markov decision processes (MDPs) \citep{sutton2018reinforcement} under adversarial corruption and with general function approximation. Distinct from standard contextual bandits or MDPs, we assume that the reward functions and transition kernels are subject to adversarial attacks before they are revealed to the agent at each round. This corruption concept admits model misspecification as a special case, and is complementary to the non-corruption model. It has broad applications such as autonomous vehicles misled by adversarially hacked navigation systems or contaminated traffic signs \citep{eykholt2018robust}; recommendation systems tricked by adversarial comments to produce incorrect rankings \citep{deshpande2012linear}. Meanwhile, general function approximation has gained considerable attention since practical algorithms usually approximate true value functions or policies by a general function class $\cF$ (e.g., neural networks) to handle the numerous states in modern large-scale reinforcement learning.

The goodness of an algorithm for learning a contextual bandit or an MDP is often measured by the notion of \emph{regret}, which is defined as the cumulative suboptimality compared with the best policy. In the standard setting without corruption, statistically and computationally efficient algorithms have been developed for contextual bandits \citep{foster2020beyond, zhang2022feel} and MDPs \citep{wang2020reinforcement, kong2021online} under nonlinear function approximation. However, the performance of these algorithms can severely degrade when the underlying models are corrupted by an adversary \citep{lykouris2018stochastic}. Therefore, designing algorithms that are robust to adversarial corruption has attracted tremendous attention \citep{bogunovic2021stochastic, zhao2021linear, lee2021achieving, ding2022robust, wei2022model, he2022nearly}. In the linear contextual bandit setting with $T$ rounds, we denote the cumulative amount of corruption by $\zeta := \sum_{t=1}^T \zeta_t$. In this setting, \citet{zhao2021linear} and \citet{ding2022robust} proposed a variant of the OFUL algorithm \citep{abbasi2011improved} which achieved a regret of $\tilde{\cO}\left(\sqrt{T} + \zeta \sqrt{T}\right)$, where other dependencies are omitted for simplicity. This \textit{multiplicative} dependence on the corruption level $\zeta$ to $T$ was undesirable as the dependence on $T$ degrades whenever $\zeta=\omega(1)$. \citet{wei2022model} developed the COBE+VOFUL algorithm based on the model selection framework\footnote{The adversary in \citet{wei2022model} corrupts the environment before the agent makes decisions and they define a maximization over the decision set as $\zeta$. The details are discussed in Remark \ref{rmk:corruption}.}, which achieves an $\tilde{\cO}\left(\sqrt{T} + \zeta\right)$ regret but with a sub-optimal dependence on the feature dimension. Moreover, COBE+VOFUL was known to be computationally inefficient. The gap to the minimax lower bound was later closed by \citet{he2022nearly}, where a novel uncertainty-weighted regression technique was designed to encourage estimators' reliance on the samples with low uncertainties. Nevertheless, similar results have not been established for contextual bandits with nonlinear function approximation.

We tackle the nonlinear contextual bandit with corruption by designing an efficient algorithm with desired additive dependence on the corruption level $\zeta$. Building on recent advances in linear contextual bandits \citep{he2022nearly}, we adapt the uncertainty-weighted regression to general function approximation. However, two challenges arise from this extension: (i) how to design sample-dependent weights with general function approximation? (ii) how to control the uncertainty level in the presence of weighting? To address the first challenge, we design an uncertainty estimator that compares the error on a newly arrived data point with the training error on the historical dataset. For the latter, we develop a novel analysis technique to control the sum of weighted uncertainty when the problem has a low eluder dimension \citep{russo2013eluder}. When specialized to the linear setting, the proposed algorithm enjoys a regret bound matching the nearly optimal result in \citet{he2022nearly}.

We extend our algorithm to the MDP setting under general function approximation by combining the uncertainty-weighted regression with the $\cF$-target Least-Squares Value Iteration ($\cF$-LSVI) \citep{wang2020reinforcement}. We assume that the adversary could attack both the rewards and transition kernels, where the amount of corruption is measured by the change of the Bellman operator (formally defined in Section~\ref{sec:mdp_corruptions}). In addition to the challenges arising in the bandit case, we also encounter stability issues of the exploration bonus in the analysis of backward iterations, meaning that the statistical complexity (i.e., the covering number) of the bonus class could be extremely high. When the function class could be embedded into a (possibly infinite-dimensional) Hilbert space, we show that the notion of an effective dimension could directly control the covering number. Moreover, for the problems with a low eluder dimension, we adopt the sub-sampling technique \citep{wang2020reinforcement, kong2021online} to select a core set for the uncertainty estimation. The use of weighted regression also brings distinct challenges to the procedure. To the best of our knowledge, the only existing result for MDPs with general function approximation is \citet{wei2022model}, whose corruption formulation is most consistent with ours. However, their algorithm utilizes fundamentally different ideas and suffers from a multiplicative dependence on the corruption level $\zeta$. We defer a detailed comparison to Section~\ref{sec:related_work}. To summarize, our contributions are threefold.
\begin{itemize}
    \item For contextual bandits under adversarial corruption and with general function approximation, we propose an algorithm based on optimism and uncertainty-weighted regression, which is computationally efficient when the weighted regression problem can be efficiently solved. We characterize the uncertainty of each data point by its diversity from history samples. Moreover, we demonstrate that the total level of uncertainty in the weighted version can be controlled by the eluder dimension \citep{russo2013eluder}.
    \item When extending to MDPs with general function approximation, we apply the uncertainty-weighting technique to the $\cF$-LSVI \citep{wang2020reinforcement} and develop an algorithm, which adds a bonus for every step and establishes optimism in the backward iteration. This algorithm is computationally efficient as long as the weighted regression problem is solvable \citep{wang2020reinforcement}. To address the stability issue of the bonus class, for a Hilbert space, we show that its covering number is bounded by its effective dimension; for general space, we adapt the sub-sampling technique \citep{wang2020reinforcement, kong2021online} but with new analysis techniques to handle the sample-dependent weights. 
    \item In terms of regret bounds, for contextual bandits, our algorithm enjoys a regret bound of $\tilde{\cO}(\sqrt{T\dim\ln N} + \zeta\dim)$, where $T$ is the number of rounds, $\dim$ is the eluder dimension of the function space, $N$ is the covering number, and $\zeta$ is the cumulative corruption. This result nearly matches the bound for linear models in known and unknown $\zeta$ cases. For MDPs, the proposed algorithm guarantees an $\tilde{\cO}(\sqrt{TH\dim\ln N} + \zeta\dim)$ regret, where $H$ is the episode length, $\dim$ is the sum of the eluder dimension for $H$ function spaces. This is the first result that has an additive dependence on $\zeta$ to $T$ for MDPs with general function approximation.
\end{itemize}

\subsection{Related Work} \label{sec:related_work}
We have discussed most of the related works throughout the introduction, thus deferring a review of additional works to Appendix~\ref{sec:additional_related_work}. In this section, we will focus on comparing our work with that of \citet{wei2022model}, who studied MDPs with general function approximation. We highlight similarities and differences between their work and ours as follows. 

In terms of corruption models, both \citet{wei2022model} and this work consider corruptions on the Bellman operator. However, the adversary considered in this paper is slightly stronger than that of \citet{wei2022model}, so our results apply to their setting but not vice versa. See Remark~\ref{rmk:corruption} for details. In terms of algorithmic ideas, we consider general function approximation under the least-squares value iteration (LSVI) framework, where optimism is achieved by adding exploration bonuses at every observed state-action pair. We then employ a generalized uncertainty-weighted regression to improve the robustness of the proposed algorithm. The proposed algorithm is known to be computationally efficient as long as we can solve the (weighted) regression problems efficiently \citep{wang2020reinforcement, kong2021online}. On the contrary, \citet{wei2022model} consider problems with a low Bellman eluder dimension and propose COBE + GOLF, which only establishes optimism at the \textit{initial} state and the optimistic planning step of GOLF \citep{jin2021bellman} is known to be computationally intractable in general. Finally, in terms of the regret bounds, \citet{wei2022model} achieve a regret bound of $\tilde{\cO}(\sqrt{T} + \zeta^r)$, where $\zeta^r$ quantifies the square norm of corruption before the observation of actions (refer to Remark \ref{rmk:corruption} for details). Their bound reduces to $\zeta^r = \cO(\sqrt{T\zeta})$ in the worst case, since they assume the boundness of corruption at each round $\zeta_t=\cO(1)$. They state that the boundness is without loss of generality since they can reduce the problem to bounded corruption case by projecting the abnormal rewards back to their range. In comparison, our algorithm achieves a regret bound of $\tilde{\cO}(\sqrt{T} + \zeta)$ when the corruption level $\zeta$ is known. This additive dependence of $\zeta$ is desirable. When the corruption level is unknown, we refer readers to Section~\ref{sec:unkown_level} for details. However, we remark that instances covered by the Bellman eluder dimension are broader than the eluder dimension. In this sense, our frameworks do not supersede the results of \citet{wei2022model}.

\section{Preliminaries}
In this section, we formally state the problems considered in the rest of this paper. Before continuing, we first define some notations to facilitate our discussions.

\noindent \textbf{Notations.} Given space $\cX$ and $\cA$, for any function $f:\cX \times \cA \to \mathbb{R}$, we define $f(x)=\max_{a\in\cA}f(x,a)$. We also denote the function space $\cX\times\cA\rightarrow[0,1]$ by $\cI^{[0,1]}$. We use the convention that $[n] = \{1,2,\cdots,n\}$. Sometimes we will use the shorthand notations $z = (x,a)$, $a \wedge b = \min(a, b)$, and $a \vee b = \max(a, b)$.

\subsection{Nonlinear Contextual Bandits with Corruption}
In a contextual bandit problem with $T$ rounds, the agent observes a \textit{context} $x_t \in \cX$ at each round $t \in [T]$ and takes action $a_t \in \cA$. After the decision, the environment generates a reward $r_t(x_t,a_t) = f_*(x_t,a_t) + \epsilon_t(x_t, a_t)$, where $\epsilon_t(x_t,a_t)$ is a mean-zero noise. Finally, the agent receives the reward $r_t$, and the next step begins. To model the adversarial corruption, assume that we have access to a known function class $\cF:\cX \times \cA \to \mathbb{R}$ but $f_* \notin \cF$ due to corruption. Specifically, we introduce the following notion of \textit{cumulative corruption} for contextual bandits.
\begin{definition}[Cumulative Corruption for Bandits]\label{def:mis_bandit}
The cumulative corruption is of level $\zeta$ if there exists $f_b\in\cF$ such that for any data sequence $\{(x_t,a_t)\}_{t\in[T]}$, we have:\footnote{We may also use the definition $r_t(x_t,a_t)=f_b(x_t,a_t) + \zeta_t + \epsilon_t$ with $\sum_{t=1}^T|\zeta_t|\le\zeta$.}
\$
\sum_{t=1}^T \zeta_t \leq \zeta, \qquad \zeta_t = |f_*(x_t,a_t) - f_b(x_t,a_t)|.
\$
\end{definition}
\noindent This corruption model is consistent with that in \citet{he2022nearly} and \citet{bogunovic2021stochastic} when specialized to linear function approximation. However, our definition further allows for general nonlinear function approximation.

\begin{remark}\label{rmk:corruption}
The adversary in our formulation is slightly stronger than in some prior works such as \citet{lykouris2018stochastic,zhao2021linear,wei2022model}, where the amount of corruption at round $t$ is determined before the decision $a_t\in\cA$: $\tilde \zeta=\sum_{t=1}^T\max_{a\in\cA}\zeta_t(x_t,a)$. Hence, $\tilde{\zeta}$ is larger than $\zeta$ in Definition \ref{def:mis_bandit}, which is chosen after the observation of action $a_t$. In addition to the corruption $\tilde{\zeta}$, these prior works consider another notion of cumulative corruption level: $\zeta^r= (T\sum_{t=1}^T\max_{a\in\cA}\zeta_t^2(x_t,a))^{1/2}$. We will also compare the proposed algorithms with the existing results under this setting in Section~\ref{sec:results}.
\end{remark}

We make the following standard assumptions in the literature of contextual bandit \citep{abbasi2011improved}.
\begin{assumption}\label{as:bandit}
Suppose the following conditions hold for contextual bandits:
\begin{enumerate}
    \item For all $(f, x,a)\in \cF \times\cX\times\cA$, we have $|f(x,a)|\le 1$;
    \item At each round $t$, let $\cS_{t-1}=\{x_s,a_s,\epsilon_s\}_{s\in[t-1]}$ denote the history. The noise variable $\epsilon_t$ is conditional $\eta$-sub-Gaussian, i.e., for all $\lambda\in\rR$, $
    \ln\E[e^{\lambda\epsilon_t}|x_t, a_t, \cS_{t-1}]\le\lambda^2\eta^2/2.$
\end{enumerate}
\end{assumption}
The {\bf learning objective} in contextual bandit is to minimize the regret:
$
\reg(T) = \sum_{t=1}^T [f_*(x_t) - f_*(x_t,a_t)].
$

\subsection{Nonlinear MDPs with Corruption} \label{sec:mdp_corruptions}
We consider an episodic MDP, represented by a tuple $\mathcal{M}(\cX,\cA,H,\Pb,r)$, where $\cX$ and $\cA$ are the spaces of states and actions, $H$ is episode length, $\Pb = \{\Pb^h\}_{h \in [H]}$ is a collection of probability measures, and $r = \{r^h\}_{h \in [H]}$ is the reward function. In each episode, an initial state $x^1$ is drawn from an unknown but fixed distribution. At each step $h \in [H]$, the agent observes a state $x^h\in\cX$, takes an action $a^h\in\cA$, receives a reward $r^h$ and transitions to the next state $x^{h+1}$ with probability $\Pb^h(r^h,x^{h+1}|x^h,a^h)$. We assume that the reward is non-negative and $\sum_{h=1}^H r^h(x^h,a^h) \leq 1$ almost surely.

\noindent \textbf{Policy and value functions.} A deterministic policy $\pi$ is a set of functions $\{\pi^h:\cX\rightarrow\cA\}_{h \in [H]}$. Given any policy $\pi$, we define its Q-value and V-value functions starting from step $h$ as follows:
\begin{equation}
\begin{aligned}
Q^h_{\pi}(x,a)&=\sum_{h'=h}^H \E_{\pi}[r^{h'}(x^{h'},a^{h'})\,|\, x^h=x,a^h=a],\\
V^h_{\pi}(x)&=\sum_{h'=h}^H \E_{\pi}[r^{h'}(x^{h'},a^{h'})\,|\,x^h=x].
\end{aligned}
\end{equation}
It is well known from \citet{sutton2018reinforcement} that there exists an optimal policy $\pi^*$ such that for all $(x,a) \in \cX \times \cA$, the optimal functions $V_*^h(x):=V^h_{\pi^*}(x)=\sup_{\pi} V_\pi^h(x)$ and $Q_*^h(x,a) := Q_{\pi^*}^h(x,a)= \sup_{\pi} Q_\pi^h(x,a)$ satisfy the following Bellman optimality equation:
\begin{equation} \label{eqn:bellman_opt}
\small
\begin{aligned}
Q_*^h(x,a) = \E_{r^h,x^{h+1}}\big[r^h + \max_{a'\in\cA}Q_*^{h+1}(x^{h+1},a') \,|\, x,a\big].
\end{aligned}
\end{equation}

\noindent \textbf{Function approximation.} We approximate the Q value function by a class of functions $\cF = \cF_1 \times \cdots \times\cF_H$ where $\cF_h \subset\cI^{[0,1]}$ for all $h \in [H]$. We use the convention that $f_{H+1} = 0$ since no reward is collected at step $H+1$. Then, we define the Bellman operator $\cT^h$ on space $\cF$: 
\begin{equation*}
\begin{aligned}
    (\cT^h f^{h+1})(x^h, a^h) := \E_{r^h, x^{h+1}} \big[r^h + f^{h+1}(x^{h+1})\big | x^h,a^h\big],
\end{aligned}
\end{equation*}
and the corresponding Bellman residual:
\begin{equation}
\cE^h(f,x^h,a^h)=f^h(x^h,a^h)-(\cT^h f^{h+1})(x^h,a^h).
\end{equation}

We now generalize the concept of corruption to the MDP setting. Since most of the existing approaches for MDPs rely on the condition that $Q_*^h$ satisfies the Bellman optimality equation, we measure the amount of corruption in terms of the change of the Bellman operator, which can occur through attacks on the rewards and probability kernels. Formally, the cumulative corruption for MDPs is defined as follows.
\begin{definition}[Cumulative corruption]\label{def:mis_mdp}
The cumulative corruption is $\zeta$ if there exists a complete and compressed operator $\cT^h_b:\cI^{[0,1]}\rightarrow\cF^h$ satisfying that $|(\cT^h_b g-\cT^h_b g')(x,a)|\le\|g - g'\|_{\infty}$ for any $g,g'\in\cI^{[0,1]}$ and $(x,a)\in\cX\times\cA$, and for all $h\in[H]$ and any sequence $\{x_t^h,a_t^h\}_{t\in[T]}\subset\cX\times\cA$, we have \footnote{Since our adversary can corrupt the data after observing the actions, if we suppose that there is no corruption when computing the regret, the corruption level is the $l^1$ norm over a sequence $\{x_t^h,a_t^h\}_{t\in[T]}$ chosen by the algorithm instead of any arbitrary sequence.}
$$ 
\sup_{g\in\cI^{[0,1]}}|(\cT^hg-\cT_b^hg)(x_t^h,a_t^h)|\le \zeta_t^h, \quad 
\sum_{t=1}^T \zeta_t^h \leq \zeta.
$$
\end{definition}
One can view $\cT_b^h$ as the Bellman operator for the uncorrupted MDP and $\cT^h$ for the corrupted MDP. When $H=1$, the corruption on the Bellman operator $\cT^h$ affects only the rewards, thus reducing the problem to that of a corrupted contextual bandit. This corruption on the Bellman operator has also been considered in \citet{wei2022model}, and it is a significant extension from the misspecified linear MDP in \citet{jin2020provably}. First, instead of assuming that the corruptions in the reward function and transition kernel are uniformly bounded for all rounds, we allow for non-uniform corruption and only require a bound on the cumulative corruption. Furthermore, when specialized to linear function approximation, our corruption on the Bellman operator subsumes the model in \citet{jin2020provably}. Finally, we consider general (nonlinear) function approximation, which is a strict generalization of the linear scenario. 

We also make the following boundedness assumption for the function class.

\begin{assumption}\label{as:mdp} For any $(h,f,x,a)\in [H] \times \cF \times \cX \times \cA$, we have $f^h(x,a) \in [0,1]$.
\end{assumption}
Suppose that we play the episodic MDP for $T$ episodes and generate a sequence of policy $\{\pi_t\}_{t\in [T]}$. The {\bf learning objective} is to minimize the cumulative regret:
$
\reg(T)= \sum_{t=1}^T\left[V_*^1(x_t^1)-V_{\pi_t}^1(x_t^1)\right].
$

\subsection{Eluder Dimension and Covering Number}
To measure the complexity of a general function class $\cF$, \citet{russo2013eluder} introduced the notion of the eluder dimension. We start with the definition of $\epsilon$-dependence. 
\begin{definition}[$\epsilon$-dependence]
A point $z$ is $\epsilon$-dependent on a set $\cZ$ with respect to $\cF$ if any $f,g\in\cF$ such that $\sqrt{\sum_{z_i\in\cZ}(f(z_i)-g(z_i))^2}\le\epsilon$ satisfies $|f(z)-g(z)|\le\epsilon$.
\end{definition}
\noindent Intuitively, the $\epsilon$-dependence means that if any two functions in a given set are relatively consistent on the historical dataset $\cZ$, their predictions on $z$ will also be similar. Accordingly, we say that variable $z$ is $\epsilon$-independent of $\cZ$ with respect to $\cF$ if $z$ is not $\epsilon$-dependent on $\cZ$.

\begin{definition}[Eluder Dimension]\label{def:eluder_dim}
For $\epsilon>0$ and a function class $\cF$ defined on $\cX$, the $\epsilon$-eluder dimension $\dim_E(\cF,\epsilon)$ is the length of the longest sequence of elements in $\cX$ such that for some $\epsilon'\ge\epsilon$, each element is $\epsilon'$-independent of its predecessors.
\end{definition}

When $\cF$ is a set of generalized linear functions, its eluder dimension is linearly dependent on the dimension of the feature map \citep{russo2013eluder}. \citet{li2021eluder} demonstrate that the eluder dimension family is strictly larger than the generalized linear class. \citet{osband2014model} prove that the quadratic function class $\{\phi(x,a)^\top \Lambda \phi(x,a): \Lambda \in \rR^{d\times d}\}$ has an eluder dimension of $\cO\left(d^2 \ln(1/\epsilon)\right)$.


For infinite function classes, we require the notion of $\epsilon$-cover and covering number. We refer readers to standard textbooks (e.g., \citet{wainwright2019high, TZ23-lt}) for more details.
\begin{definition}[$\epsilon$-cover and covering number] Given a function class $\cF$, for each $\epsilon > 0$, an $\epsilon$-cover of $\cF$ with respect to $\norm{\cdot}_\infty$, denoted by $\mathcal{C}(\cF, \epsilon)$, satisfies that for any $f \in \cF$, we can find $f' \in \mathcal{C}(\cF, \epsilon)$ such that $\norm{f - f'}_\infty \leq \epsilon$. The $\epsilon$-covering number, denoted as $N(\epsilon, \cF)$, is the smallest cardinality of such $\mathcal{C}(\cF, \epsilon)$. 
\end{definition}

\section{Algorithms} \label{sec:algorithm}
In this section, we discuss the limitations of existing methods and the technical challenges of extending corruption-robust algorithms to general function approximation, thus motivating our approaches. For notation simplicity, we use the shorthand notation $z = (x,a)$ and $Z_h^t := \{z_h^s\}_{s=1}^t$.

\subsection{Bandits with General Function approximation} \label{sec:bandit}

For nonlinear contextual bandit, the most prominent approach is based on the principle of ``Optimism in the Face of Uncertainty'' (OFU). In this approach, the core step in each round $t\in[T]$ is to construct an appropriate confidence set $\cF_t$ so that the approximator $f_b$ lies in $\cF_t$ with high probability.  
The standard approach is to solve the following least-squares regression problem:
\begin{equation} \label{eqn:naive_regression}
    \hat{f}_t = \argmin_{f \in \cF_{t-1}} \sum_{s=1}^{t-1} \big(f(z_s)-r_s\big)^2, 
\end{equation}
and construct set $\cF_t = \{f \in \cF_{t-1}: \sum_{s=1}^{t-1} |f(z_s)-\hat{f_t}(z_s)|^2 \leq \beta_t^2\}$ where $\beta_t^2 = \cO(\ln N)$ is set to be the log-covering number (as stated in Proposition 2 of \citep{russo2013eluder}) such that $f_b \in \cF_t$ with high probability. 
However, due to the adversarial corruption, the collected samples are now from $f_* \notin \cF$. To ensure that $f_b \in \cF_t$ in this case, the confidence radius needs to be enlarged. Specifically, we determine $\beta_t$ by examining $\sum_{s=1}^{t-1}(f_b(z_s) - \hat{f}_t(z_s))^2$, which suffers from an additional cross-term $\sum_{s=1}^{t-1}|f_b(z_s)-\hat f_t(z_s)|\zeta_s$. To control this term, one has to set $\beta_t^2 = \cO (\zeta + \ln N)$, leading to a final regret bound of $\cO(\zeta \sqrt{T})$. For further understanding, we refer readers to the intuitive explanation for the linear contextual bandit in Appendix B of \citet{wei2022model}. The multiplicative relationship between $\zeta$ and $\sqrt{T}$ is disastrous, and the regret bound becomes vacuous when $\zeta=\Omega(\sqrt{T})$.


\begin{algorithm}[htb]
	\caption{CR-Eluder-UCB}
	\label{alg:ucb_nonlinear_bandit}
	\begin{small}
		\begin{algorithmic}[1]
		    \STATE \textbf{Input:} $\lambda>0,~T,\cF$ and $\cF_0 = \cF$.
			\FOR{Stage $t=1,\dots,T$}
			\STATE Observe $x_t \in \cX$;
			\STATE Find the weighted least-squares solution $\hat{f}_t$ as in \eqref{eqn:weighted_regression};
			\STATE Find $\beta_t>0$ and construct the confidence set $\cF_t$ as in \eqref{eqn:confidence_set};
                \STATE Take the most optimistic function $f_t = \argmax_{f \in \cF_t} f(z_t)$ and choose $a_t = \argmax_{a\in \cA} f_t(x_t,a)$;
			\STATE Receive $r_t$ and set weight $\sigma_t^2$ as in \eqref{eq:weight_bandit}
			\ENDFOR
		\end{algorithmic}
	\end{small}
\end{algorithm}

To overcome this issue, we adapt the uncertainty-weighting strategy from \citet{he2022nearly}, which considers the linear function class in terms of a feature $\phi:\cX \times \cA \to \mathbb{R}^d$ and employs the weighted ridge regression:
$$
\hat{\theta}_t \leftarrow \argmin_{\theta \in \cF_{t-1}}  \sum_{s=1}^{t-1}\frac{(\theta^\top \phi(z_s) - r_s)^2}{\max(1,\frac{1}{\alpha}\|\phi(z_s)\|_{\Lambda_s^{-1}})},
$$
where $\Lambda_t = \lambda I + \sum_{s=1}^{t-1} \phi(z_s) \phi(z_s)^\top$ and $\alpha > 0$ is a tuning parameter. The weights in this equation are a truncated version of the bonus $\norm{\phi(z_s)}_{\Lambda_t^{-1}}$, which can be viewed as the uncertainty of the sample. This means that the estimation of $\hat{f}_t$ will rely more on samples with low uncertainty. The close correlation between weights and bonuses 
stems from the following considerations: when making decisions, the exploration bonus encourages uncertainty; when making estimations, the weights are used to punish uncertainty.

As a natural extension, we replace \eqref{eqn:naive_regression} with its weighted version:
\begin{equation} \label{eqn:weighted_regression}
    \hat{f}_t = \argmin_{f \in \cF_{t-1}} \sum_{s=1}^{t-1} \frac{\big(f(z_s)-r_s\big)^2}{\sigma_s^2},
\end{equation}
and we accordingly set the confidence set as 
\begin{equation}\label{eqn:confidence_set} 
    \begin{aligned}
\cF_t = \Big\{ f \in \cF_{t-1}: \sum_{s=1}^{t-1} \frac{\big(f(z_s)-\hat{f_t}(z_s)\big)^2}{\sigma_s^2} \leq \beta_t^2\Big\}.   
    \end{aligned}
\end{equation}
We then compute the most optimistic value function $f_t$ from the confidence set $\cF_t$ and follow it greedily. The algorithm is called corruption-robust eluder UCB algorithm (CR-Eluder-UCB), whose pseudo-code is given in Algorithm~\ref{alg:ucb_nonlinear_bandit}. While the algorithmic framework shares a similar spirit with the weighted ridge regression, the extension is not straightforward because the bonus and weight choices in \citet{he2022nearly} and their theoretical analysis (which will be discussed later) heavily rely on the linear structure. To handle the general function approximation, we first define the following (weighted) uncertainty estimator:\footnote{We remark that \citet{gentile2022achieving} considers a similar quantity but without weighting.} 
\begin{equation}\footnotesize
    \label{def:wdim_bandit}
    \begin{aligned} 
    D_{\lambda,\sigma,\cF_t}(Z_1^t) = 1 \wedge \sup_{f\in \cF_t}\frac{|f(z_t)-\hat{f_t}(z_t)|/\sigma_t}{\sqrt{\lambda + \sum_{s=1}^{t-1} |f(z_s)-\hat{f_t}(z_s)|^2/\sigma_s^2}},
    \end{aligned}
\end{equation}
where $Z_1^t = \{(x_s,a_s)\}_{s\in[t]}$, and we omit the superscript since bandits only have one step. Intuitively, it measures the degree to which the prediction error on $(x_t,a_t)$ exceeds the historical error evaluated on $Z_1^{t-1}$, thus serving as an uncertainty estimation of $z_t$ given the historical dataset $Z_1^{t-1}$. Accordingly, with $\alpha > 0$ as a tuning parameter, we adopt the following weights:
\begin{equation}\footnotesize
    \label{eq:weight_bandit}
    \begin{aligned} 
\sigma_t^2 = 1 \vee \sup_{f\in \cF_t} \frac{|f(z_t)-\hat{f_t}(z_t)|/\alpha}{\sqrt{\lambda + \sum_{s=1}^{t-1}\big(f(z_s)-\hat{f_t}(z_s)\big)^2/\sigma_s^2}},
    \end{aligned}
\end{equation}
where $a \vee b := \max(a, b)$. When the problem has a low eluder dimension, regular regression, i.e., $\sigma_s \equiv 1$, has a nonlinear analog to the elliptical potential lemma \citep{abbasi2011improved} for linear problems. 
\#\label{eq:sum_eluder_like _quan_dimE}
\sup_{Z_1^T} \sum_{t=1}^T D_{\lambda, \sigma, \cF}^2(Z_1^t) = \tilde{\cO}\big(\dim_{E}(\cF, \lambda)\big).\#
However, when a sequence of weights $\sigma_s \geq 1$ is introduced, the situation becomes quite different. In linear settings, introducing the weights is straightforward. For any ball $\|\theta\|_{2} \leq R$ with radius $R > 0$, the weighted version $\theta/\sigma_s$ also lies in the original ball of parameters. This result heavily relies on the linear structure where the function class is closed under scaling. In contrast, for weighted regression with general nonlinear function approximation, we need to explicitly control the weights in order to control the uncertainty level, which requires new analysis techniques. The generalized result of \eqref{eq:weight_bandit} is Lemma \ref{lm:Relationship_between_Eluder_Dimensions}, which also handles the case of $\sigma_s \geq 1$. This result is essential for the theoretical analysis of nonlinear function approximation. 

\subsection{MDPs with General Function Approximation} \label{sec:mdp_alg}
In this subsection, we extend Algorithm~\ref{alg:ucb_nonlinear_bandit} to MDPs with general function approximation and adversarial corruption. We adopt the $\cF$-LSVI algorithm, which constructs the Q-value function estimations $f_t^h$ backward from $h=H$ to $h=1$, with the initialization $f_t^{H+1} = 0$. Specifically, suppose that we have constructed $f_t^{h+1}$ and proceed to construct $f_t^h$. We reduce the problem to the bandit scenario by defining auxiliary variables $y_s^h = r_s^h + f_t^{h+1}(x_s^{h+1})$ for $s \leq t-1$. Then, we solve the following weighted least-squares problem to approximate the Bellman optimality equation~\eqref{eqn:bellman_opt}:
\begin{equation} \label{eqn:least_square_mdp}
    \hat{f}_t^h=\argmin_{f^h\in\cF_{t-1}^h}\sum_{s=1}^{t-1}\frac{(f^h(x_s^h,a_s^h)-y_s^h)^2}{(\sigma_s^h)^2}.
\end{equation}
We create the confidence set $\cF_t^h = \{f^h \in \mathcal{F}_{t-1}^h: \lambda+\sum_{s=1}^{t-1}(f^h(z_s^h)-\hat{f}_t^h(z_s^h))^2 /(\sigma_s^h)^2 \leq(\beta_t^h)^2\}$ such that $\cT_b^h f_t^{h+1} \in \cF_t^h$ with a high probability. Then, we choose the most optimistic estimation by adding a bonus function to the least-square solution. Specifically, for any $z \in \cX \times \cA$, 
\begin{equation} \label{eqn:confidence_set_mdp}
    f_t^h(z) = 1 \wedge \big(\hat{f}^h_t(z) + \beta_t^h b_t^h(z)\big),
\end{equation}
where we also clip the estimation to the valid range. It remains to determine the bonus function $b_t^h$, which serves as an uncertainty estimator. A natural candidate is 
\#\label{eq:bonus_mdp}
\sup_{f^h\in\cF_t^h}\frac{|f^h(z)-\hat{f}_t^h(z)|}{\sqrt{\lambda+\sum_{s=1}^{t-1}|f^h(z_s^h)-\hat{f}_t^h(z_s^h)|^2/(\sigma_s^h)^2}}.
\#
However, one may not directly use such a bonus. Compared with the bandit setting, since $f_t^{h+1}$ is computed by the least-squares problem in later steps, it depends on $\{x_s^h,a_s^h\}_{s=1}^{t-1}$ because the later state is influenced by the previous decisions. Therefore, concentration inequality cannot be applied directly due to the measurability issue. The standard approach to address the issue is to establish a uniform concentration over an $\epsilon$-covering of the function space $f_t^{h+1}$, which depends on $\cF^{h+1}$ and the space of bonus $b_t^{h+1}$. The bonus \eqref{eq:bonus_mdp} is ``unstable'' in the sense that the statistical complexity (covering number) of the bonus space could be extremely high because it is obtained from an optimization problem involving $Z_1^T$ whose size can be as large as $T$ in the worst case. 

\begin{algorithm}[htb]
	\caption{CR-LSVI-UCB}
	\label{alg:ucb_nonlinear_mdp}
	\begin{small}
		\begin{algorithmic}[1]
		    \STATE \textbf{Input:} $\lambda>0,~T,~\{\cF^h\},~\{\cB^h(\lambda)\}$ and $\cF_0^h=\cF^h$
			\FOR{Episode $t=1,\dots,T$}
			\STATE Receive the initial state $x_t^1$
			\STATE Let $f_t^{H+1}=0$
			\FOR{Step $h=H,\ldots,1$}
			\STATE Find the weighted least-squares solution $\hat{f}_t^h$ as in \eqref{eqn:least_square_mdp};
			\STATE Set $f_t^h(x^h,a^h)$ as in \eqref{eqn:confidence_set_mdp} with $\beta_t^h>0$, and bonus function $b_t^h(\cdot,\cdot)$ as in \eqref{eq:bonus_mdp}.
                \ENDFOR
			\STATE Let $\pi_t$ be the greedy policy of $f_t^h$ for each step $h\in[H]$;
			\STATE Play policy $\pi_t$ and observe trajectory $\{(x_t^h,a_t^h,r_t^h)\}_{h=1}^H$;
			\STATE Set weight $\sigma_t^h$ as in \eqref{eq:weight_mdp} for all $h\in[H]$;
			\ENDFOR
		\end{algorithmic}
	\end{small}
\end{algorithm}

To facilitate our analysis, we assume that we have access to a bonus class $\cB^{h+1}(\lambda)$ that has a mild covering number and can approximate \eqref{eq:bonus_mdp}. While we may encounter an approximation error, we assume that the equality holds for simplicity because it suffices to illustrate the idea that we want to design a corruption-robust algorithm for MDPs with general function approximation. We will provide examples in Appendix~\ref{s:subsample} where the approximation error is usually negligible under suitable conditions. In summary, we assume that for any sequence of data $Z_h^{t-1}=\{(x_s^h,a_s^h,r_s^h)\}_{s=1}^{t-1}$, we can find a $b_t^h \in \cB^h(\lambda)$ that equals \eqref{eq:bonus_mdp}. With this bonus space, we are ready to present Corruption-Robust LSVI-UCB (CR-LSVI-UCB) in Algorithm \ref{alg:ucb_nonlinear_mdp}. 
We define the confidence interval quantity for MDPs 
$D_{\lambda,\sigma^h,\cF_t^h}(Z_h^t)$ as:\footnote{We note that a similar treatment of nonlinear LSVI-UCB appeared in \citet{TZ23-lt} but without weighting.}
\begin{equation} \footnotesize
\label{def:wdim_mdp}
    \begin{aligned} 
    1 \wedge \sup_{f^h\in\cF_t^h}\frac{|f^h(z^h_t)-\hat{f}_t^h(z^h_t)|/\sigma_t^h}{\sqrt{\lambda+\sum_{s=1}^{t-1}|f^h(z_s^h)-\hat{f}_t^h(z_s^h)|^2/(\sigma_s^h)^2}}.
    \end{aligned}
\end{equation} 
Note that the $D_{\lambda,\sigma^h,\cF_t^h}(Z_h^t)$ and the bonus $b_t^h$ mainly differ in a factor of weight in the iteration $t$. Therefore, they are almost identical for the regular regression. However, as we employ weighted regression in the algorithm, they are not identical. Specifically, $b_t^h(z)$ is the bonus used for the agent to ensure optimism, while $D_t^h(z)$ is the ratio between the weighted prediction error and weighted in-sample error, which is used in the theoretical analysis to bound the regret (see Eqn. \eqref{eqn:regret_mdp} for details).

Similarly, with $\alpha$ as the tuning parameter, the sample-dependent weights are chosen to be 
\begin{equation}\footnotesize
    \label{eq:weight_mdp}
    \begin{aligned} 
(\sigma_t^h)^2= 1 \vee \sup_{f^h\in\cF_t^h}\frac{|f^h(z_t^h)-\hat{f}_t^h(z_t^h)|/\alpha}{\sqrt{\lambda+\sum_{s=1}^{t-1}(f^h(z_s^h)-\hat{f}_t^h(z_s^h))^2/(\sigma_s^h)^2}}.
    \end{aligned}
\end{equation}

\section{Main Results} \label{sec:results}
In this section, we establish the main theoretical results.
\subsection{Bandits with General Function Approximation}
\begin{theorem}\label{th:bandit}
Suppose that Assumption \ref{as:bandit} holds. For any cumulative corruption $\zeta>0$ and $\delta\in(0,1)$, we take the covering parameter $\gamma=1/(T\zeta)$, the eluder parameter $\lambda=\ln(N(\gamma,\cF))$, the weighting parameter $\alpha=\sqrt{\ln(N(\gamma,\cF))}/\zeta$ and the confidence radius 
$$
\beta=c_{\beta}(\alpha\zeta+\sqrt{\ln(N(\gamma,\cF)/\delta)}+\sqrt{c_0}),
$$
where $c_0=\sqrt{\eta^2\ln(2/\delta)}$ and $c_{\beta}>0$ is an absolute constant. Then, with probability at least $1-\delta$, the cumulative regret with $T$ rounds is bounded by
$$
\footnotesize
\begin{aligned}
\Tilde{\cO}\bigg(\sqrt{T\dim_E\Big(\cF,\sqrt{\frac{\lambda}{T}}\Big)\ln(N(\gamma,\cF))} + \zeta\dim_E\Big(\cF,\sqrt{\frac{\lambda}{T}}\Big)\bigg).
\end{aligned}
$$
\end{theorem}
\textbf{Interpretation.} This theorem asserts that for contextual bandits under adversarial corruption and with general function approximation, Algorithm~\ref{alg:ucb_nonlinear_bandit} achieves an additive dependence on the cumulative corruption $\zeta$ as desired. To interpret the result, we consider linear contextual bandits with corruption, where $\dim_E(\cF,\lambda/T)=d$ and $\ln(N(\gamma,\cF))=\tilde{\cO}(d)$. This implies a regret bound of $\tilde{\cO}(d\sqrt{T}+d\zeta)$, matching that of \citet{he2022nearly}. According to the lower bound in \citet{he2022nearly} and two existing lower bound results \citep{lattimore2020bandit,bogunovic2021stochastic}, our regret is minimax optimal up to logarithmic factors. Particularly when $\zeta=\cO(\sqrt{T})$ and $\dim_E(\cF,\lambda/T)=\tilde{\cO}(\ln(N(\gamma,\cF))$, the first term dominates and matches the regret bound in the uncorrupted setting. Moreover, our regret is sublinear when $\zeta=o(T)$.


\subsection{MDPs with General Function Approximation}
\begin{theorem}\label{th:mdp}
Suppose that Assumption \ref{as:mdp} holds. For any cumulative corruption $\zeta>0$ and $\delta\in(0,1)$, we take the covering parameter $\gamma=1/(T\beta\zeta)$, the
eluder parameter $\lambda=\ln (N_T(\gamma))$, the weighting parameter $\alpha=\sqrt{\ln N_T(\gamma)}/\zeta$ and the confidence radius $
\beta_t^h=\beta(\delta)=c_{\beta}(\alpha\zeta+\sqrt{\ln (HN_T(\gamma)/\delta)})$, where 
$$\footnotesize
\begin{aligned}
N_T(\gamma)=\max_h N\Big(\frac{\gamma}{T},\cF^h\Big)\cdot N\Big(\frac{\gamma}{T},\cF^{h+1}\Big)\cdot N\Big(\frac{\gamma}{T},\cB^{h+1}(\lambda)\Big),
\end{aligned}
$$ 
and we remark that the coverings are all with respect to the $\norm{\cdot}_\infty$. Then, with probability at least $1-\delta$, the cumulative regret with $T$ rounds is bounded by
\begin{equation} \label{eqn:regret_mdp}
\footnotesize
\begin{aligned}
\Tilde{\cO}\bigg(\sqrt{TH\ln(N_T(\gamma))\dim_E\Big(\cF,\sqrt{\frac{\lambda}{T}}\Big)} + \zeta\dim_E\Big(\cF,\sqrt{\frac{\lambda}{T}}\Big)\bigg).
\end{aligned}
\end{equation}
\end{theorem}
\textbf{Interpretation.} This theorem guarantees that for corrupted models in MDPs with general function approximation, Algorithm~\ref{alg:ucb_nonlinear_mdp} achieves an additive dependence on $\zeta$. In particular, for linear MDPs, we can obtain that $\ln N_T(\gamma)=\Tilde{\cO}(d^2)$ and 
$$
\begin{aligned}
\dim_E\Big(\cF,\sqrt{\frac{\lambda}{T}}\Big) = \sum_{h=1}^H \dim_E\Big(\cF^h, \frac{\lambda}{T}\Big) = \Tilde{\cO}\big(Hd\big).
\end{aligned}
$$ 
It follows from Theorem \ref{th:mdp} that $\reg(T)=\Tilde{\cO}\big(\sqrt{TH^2d^3}+\zeta Hd\big)$, where the first term matches the bound of LSVI-UCB in the non-corrupted setting \citep{jin2020provably}. We compare our result with regret for misspecified linear MDPs in \citet{jin2020provably}, which is a special case of our corrupted setting. By taking $\zeta=T\zeta'$, where $\zeta'$ is the uniform bound for misspecification, our algorithm achieves a regret bound of $\tilde{\cO}\big(\sqrt{TH^2d^3} + \zeta' dHT\big)$, which is consistent with that in \citet{jin2020provably}. We note that our corruption-independent term suffers from a $\sqrt{d}$ amplification due to the uniform concentration used for handling temporal dependency, which also happens for other LSVI-based works \citep{jin2020provably, wang2020reinforcement, kong2021online}.

\subsection{Unknown Corruption Level} \label{sec:unkown_level}
This section discusses the solution for unknown cumulative corruption $\zeta$ by following \citet{he2022nearly}. We estimate $\zeta$ by a tuning parameter $\bar \zeta$ and replace $\zeta$ in the threshold parameter $\alpha$ by $\bar \zeta$ for the bandit and MDP models. The bound for MDPs is shown in the following theorem.
\begin{theorem}\label{th:mdp_unknown} Set  $\beta_t^h=c_{\beta}\sqrt{\ln (HN_T(\gamma)/\delta)}$ and $\alpha=\sqrt{\ln N_T(\gamma)}/\bar\zeta$, and keep other conditions the same as Theorem \ref{th:mdp}. If $0\le \zeta \le \bar\zeta$, with probability at least $1-\delta$, we have the same bound in \eqref{eqn:regret_mdp} except that we replace $\zeta$ with $\bar{\zeta}$. If $\zeta>\bar\zeta$, we have $\reg(T)=\Tilde{\cO}(T)$.
\end{theorem}
The detailed proof is provided in Appendix \ref{s:Unknown_Corruption_Level}, and we can have similar results for bandits. 

\textbf{Comparison under a weak adversary.} We compare our result with the COBE framework \citep{wei2022model} under the weak adversary scenario introduced in Remark \ref{rmk:corruption}. The key property of COBE is that under a weak adversary, it can convert any algorithm $\mathbf{Alg}$ with a known corruption level $\zeta$ to a new COBE+$\mathbf{Alg}$ that achieves the same regret bound in the unknown $\zeta$ case.

For general function approximation,  COBE + GOLF achieves a regret bound of $\cO\big(\sqrt{T}+\zeta^r\big)$, which degenerates to $\tilde{\cO}(\sqrt{T\zeta})$ in the worst case. Thus, by choosing $\bar{\zeta}=\Theta(\sqrt{T})$, our regret bound outperforms theirs for all $\zeta = \cO(\sqrt{T})$. Moreover, in this case, our regret bound is order-optimal due to the lower bound for the uncorrupted linear MDP. When specialized to linear setting, their proposed COBE + VARLin attains a regret bound of $\cO\big(Hd^{4.5}\sqrt{T}+Hd^4\zeta\big)$, which is better than ours when $\zeta>\sqrt{T}$ in terms of the order of $T$ but is also of a worse dependence on the feature dimension $d$. 

Particularly, we can take the base algorithm $\mathbf{Alg}$ in COBE as Algorithm~\ref{alg:ucb_nonlinear_mdp}. By combining Theorem 3 of \citet{wei2022model} and Theorem~\ref{th:mdp}, we achieve a regret bound of $\tilde{\cO}(\sqrt{TH\dim\ln N} + \zeta\dim)$ in the unknown $\zeta$ case where $N$ denotes the covering number and $\dim$ denotes the sum of eluder dimension. This bound reduces to $\cO\big(d^{1.5}H\sqrt{T} + dH\zeta \big)$ in the linear setting, which is better than that of COBE + VARLin. Nevertheless, this bound only applies to the weak adversary. While our results also apply to the strong adversary as mentioned in Remark \ref{rmk:corruption}.



\section{Proof Sketch} 
This section provides an overview of the proof to highlight the technical challenges and novelties. We focus on MDPs and defer the details of contextual bandits to Appendix~\ref{sec:pf_th:bandit} because the main ideas are similar. The proof strategy of Theorem \ref{th:mdp} follows three steps.

\textbf{Step I: Regret under optimism.}
If for all $(t,h)\in[T]\times[H]$, the Bellman backup $\cT_b^hf_t^{h+1}$ belongs to the confidence set $\cF_t^h$, we can show that  
$\sum_{t=1}^T V_*^1(x_t^1)\le \sum_{t=1}^T f_t^1(x_t^1) + H\zeta$. This optimism allows the regret to be upper bounded as
$$
\begin{aligned}
    \reg(T) &\leq \sum_{t=1}^T[f_t^1(x_t^1)-V_{\pi_t}^1(x_t^1)] + H\zeta \\
    &=  \sum_{t=1}^T\sum_{h=1}^H\E_{\pi_t}\cE^h(f_t,x_t^h,a_t^h) + H \zeta,
\end{aligned}
$$
where the equality is from Lemma~\ref{lm:bellman_error} and the subscript denotes the distribution induced by executing $\pi_t$. This is the main technical reason why we consider the optimism-based algorithm 
. The details are shown in Appendix \ref{ssec:step_I}.

\textbf{Step II: Sharper confidence radius for optimism.}
To ensure that the optimism in step I is achieved with a high probability, we determine the confidence radius by bounding the in-sample gap between $\cT_b^hf_t^{h+1}$ and $\hat f_t^h$, i.e., $\sum_{s=1}^{t-1}((\hat f_t^h - (\cT_b^hf_t^{h+1}))(x_s^h,a_s^h))^2/(\sigma_s^h)^2$. This can be done by standard martingale concentration inequalities and a union bound. The key observation here is that with our uncertainty-based weight design, the cross-term can be controlled as follows:
\$
\sum_{s=1}^t\frac{(\hat{f}_t^h(x_s^h,a_s^h)-(\cT_b^hf_t^{h+1})(x_s^h,a_s^h))\zeta_s^h}{(\sigma_s^h)^2}\le 2\alpha\zeta\sup_{s<t}\beta_s^h.
\$
In Lemma \ref{lm:ucb_mdp}, we demonstrate that the optimism is achieved with $(\beta_t^h)^2 = \cO(\ln N_T^h(\lambda))$. In addition to the covering number, the regular regression also requires an $\Omega(\zeta)$ confidence radius as discussed in Section~\ref{sec:bandit}. This sharper bound for the estimation error illustrates the power of weighted regression and is the key to our improvement. The detailed proof is presented in Appendix \ref{ssec:step_II}.


\textbf{Step III: Bound the sum of bonus.} If optimism is achieved, 
we can further upper bound the regret by
\begin{equation} \label{eqn:sketch_2}
\footnotesize
\begin{aligned}
\reg(T) \leq \Tilde{\cO}&\bigg(\Big[TH\beta \sum_{h=1}^H\sup_{Z_h^T}\sum_{t=1}^T D^2_{\lambda,\sigma^h,\cF_t^h}(Z_h^t)\Big]^{1/2}\\
&~~ + \zeta\sum_{h=1}^H\Big(1+\sup_{Z_h^T}\sum_{t=1}^T D^2_{\lambda,\sigma^h,\cF_t^h}(Z_h^t)\Big)\bigg).
\end{aligned}
\end{equation}
It remains to handle the sum of weighted uncertainty estimators $D^2_{\lambda,\sigma^h,\cF_t^h}(Z_h^t)$. As previously discussed at the end of Section~\ref{sec:bandit}, unlike the linear setting where the weighted function still lies in the original function class, the general function class is not closed under scaling. Therefore, we must specifically deal with the weights in our analysis. To solve this problem for uncertainty-based weights, we use a novel weight-level control technique. The key insight is that we can divide the samples into different classes based on their uncertainty levels. For samples in each class, through a refined analysis, we can demonstrate that their weights are roughly the same order, thus effectively canceling each other out. This is summarized in the following lemma.
\begin{lemma}\label{lm:Relationship_between_Eluder_Dimensions}
For a function space $\cG$ and any given sequence $Z^T=\{(x_t,a_t)\}\subset\cX\times\cA$. Under Algorithm \ref{alg:ucb_nonlinear_bandit} and \ref{alg:ucb_nonlinear_mdp}, taking $D_{\lambda,\sigma,\cG}(Z^T)$ in \eqref{def:wdim_bandit}, the weight $\{\sigma_t\}_{t\in[T]}$ in \eqref{eq:weight_bandit}, $\alpha=\sqrt{\ln N}/\zeta$ and $\lambda=\ln N$, we obtain
$$
\begin{aligned}
&\sup_{Z_1^T} \sum_{t=1}^T (D_{\lambda,\sigma,\cG_t}(Z_1^t))^2\\
&\qquad\le (\sqrt{8c_0}+3)\dim_E\Big(\cG,\sqrt{\frac{\lambda}{T}}\Big)\log\Big(\frac{T}{\lambda}\Big)\ln T,   
\end{aligned}
$$
where $c_0$ is an absolute constant such that $
\lambda + \beta_t^2 \le c_0 \ln N$, and we denote $\log_2(\cdot)$ by $\log(\cdot)$ and  $N(\gamma,\cG)$ by $N$.
\end{lemma}
The detailed proofs of \eqref{eqn:sketch_2} and Lemma \ref{lm:Relationship_between_Eluder_Dimensions} are provided in Appendix \ref{ssec:step_III} and \ref{s:Relationship_between_Confidence_and_Eluder_Dimension} respectively.

\section{Conclusions}
In this paper, we study contextual bandits and MDPs in the presence of adversarial corruption and with general function approximation. We propose a CR-Eluder-UCB for contextual bandits and a CR-LSVI-UCB for episodic MDPs, respectively. The proposed algorithms are computationally efficient when the weighted regression problems can be efficiently solved \citep{wang2020reinforcement} and are based on the uncertainty-weighted least-squares regression \citep{he2022nearly}. Accordingly, we design a new uncertainty-weighted estimator for general function classes and develop novel techniques to control the sum of the weighted uncertainty level. This leads to regret bounds that depend additively on the corruption level. Specifically, for the contextual bandit problem, the CR-Eluder-UCB algorithm achieves regret bounds that nearly match the lower bound when specialized to the linear setting for both known and unknown corruption levels. Moreover, for nonlinear MDPs with general function approximation, CR-LSVI-UCB is the first algorithm that establishes an additive dependence on $\zeta$. 

We hope this work provides valuable insights into the broad applicability of the weighted least-squares regression technique in achieving improved robustness in the general function approximation setting. For future research, it would be interesting to explore the potential of combining uncertainty weighting with variance weighting to enhance further the regret bound for nonlinear MDPs.

\section{Acknowledegement}
The authors would like to thank Chen-Yu Wei for the helpful discussions. Chenlu Ye, Wei Xiong and Tong Zhang acknowledge the funding supported by the GRF 16310222 and GRF 16201320.

\bibliography{example_paper}
\bibliographystyle{icml2023}

\newpage
\appendix
\onecolumn

\section{Additional Related Work} \label{sec:additional_related_work}

\textbf{Corruption-robust bandits.} \citet{lykouris2018stochastic} first studied the multi-armed bandit problem with corruption where the reward was corrupted by $\zeta_t$ at round $t$ and the corruption level was defined as $\zeta = \sum_{t=1}^T \zeta_t$. \citet{lykouris2018stochastic} proposed an algorithm whose regret bound was of $\tilde{\cO}(\sqrt{T} \zeta)$. \citet{gupta2019better} established a lower bound, showing that linear dependence on $\zeta$ was near-optimal. Therefore, the main goal was to design a corruption-robust algorithm with a regret bound of $\reg(T) = o(T) + \cO(\zeta)$. In particular, when $\zeta$ was non-dominating, we expected the first term to approach the non-corruption counterpart. Beyond the multi-armed bandit, \citet{li2019stochastic} and \citet{bogunovic2021stochastic} considered the stochastic bandit with linear function approximation. However, their algorithms heavily relied on arm-elimination techniques, thus confined to the stochastic and finite-arm scenarios. \citet{bogunovic2021stochastic} and \citet{lee2021achieving} studied the corruption-robust linear contextual bandit with additional assumptions, including a diversity assumption on the context and a linearity assumption on the corruption. \citet{zhao2021linear} and \citet{ding2022robust} proposed a variant of OFUL \citep{abbasi2011improved} but the regret bounds were sub-optimal.  \citet{foster2020adapting} considered an algorithm based on online regression oracle but the result was also sub-optimal. More recently, \citet{wei2022model} developed a general framework to handle unknown corruption case based on model selection. The COBE + OFUL algorithm of \citet{wei2022model} achieved a regret of $\tilde{\cO}\left(\sqrt{T} + \zeta^r\right)$ where $\zeta^r$ is a different notion of corruption level. However, in the worse case, we had $\zeta^r = \cO(\sqrt{T\zeta})$, which was still a multiplicative dependence on $\zeta$. \citet{wei2022model} also proposed a  COBE + VOFUL, which enjoyed an additive dependence of the corruption level $\zeta$ but with a sub-optimal dependence on the feature dimension. Also, COBE + VOFUL was known to be computationally inefficient. The suboptimality was fully addressed in \citet{he2022nearly} for linear contextual bandits, which showed that a computationally-efficient algorithm could achieve the minimax lower bound in linear contextual bandits. The idea was to adaptively use the history samples in parameter estimation by assigning sample-dependent weights in the regression subroutine. However, whether we can design a computationally-efficient algorithm to achieve a regret of $\tilde{\cO}\left(\sqrt{T} + \zeta \right)$ remains open for nonlinear contextual bandits.

\textbf{Corruption-robust MDPs.} Most existing works studied the adversarial reward setting where an adversary corrupted the reward function at each round, but the transition kernel remained fixed. See, e.g., \citet{neu2010online, rosenberg2019online, rosenberg2020stochastic, jin2020learning, luo2021policy, chen2021finding} and reference therein. Notably, \citet{jin2020learning} and \citet{jin2020simultaneously} developed algorithms that achieved near-minimax optimal regret bound in the adversarial reward case while preserving refined instance-dependent bounds in the static case. Therefore, the adversarial reward setting was relatively well understood. When the transition kernel could also be corrupted, \citet{wu2021reinforcement} designed an algorithm for the tabular MDP whose regret scales optimally with respect to the corruption level $\zeta$. \citet{wei2022model} considered corrupted MDPs with general function approximation under a weak adversary, that determined the corruption at one round before the agent made the decision. The corruption in \citet{wei2022model} was also imposed on the Bellman operator, so was consistent with ours. They proposed COBE + GOLF based on model selection and GOLF \citep{jin2021bellman}, which achieved a regret bound of $\tilde{\cO}(\sqrt{T} + \zeta^r)$, in the unknown corruption level case under the weak adversary. However, we remark that instances covered by the Bellman eluder dimension are broader than eluder dimension. In this sense, our frameworks do not supersede the results of \citet{wei2022model}.

\section{Proof for Contextual Bandits}\label{sec:pf_th:bandit}
This section presents the proof of Theorem \ref{th:bandit}.

\subsection{Step I: Regret under Optimism}
\begin{lemma}\label{lm:Regret_under_Optimism_bandit}
Assume that $f_b\in\cF_t$ for all $t\in[T]$. Then, we have
$$
\reg(T)\le 2\zeta + \sum_{t=1}^T \left(f_t(x_t, a_t) - f_b(x_t,a_t)\right).
$$
\end{lemma}
\begin{proof}
According to Algorithm \ref{alg:ucb_nonlinear_bandit}, for all $t\in[T]$, since $f_b\in\cF_t$, we have for $a_t^b=\argmax_{a\in\cA}f_b(x_t,a)$,
\#\label{eq:fb_xt_a}
f_b(x_t, a_t^b)\le f_t(x_t,a_t^b)\le f_t(x_t,a_t),
\#
where the first inequality is due to the optimism of $f_t$, and the second inequality uses $a_t=\argmax_{a\in\cA}f_t(x_t,a)$. Therefore,
the regret is bounded as follows:
\$
    \reg(T) &= \sum_{t=1}^T \left(\max_{a\in\cA}f_*(x_t,a) - f_*(x_t, a_t)\right) \\
    &\le \sum_{t=1}^T \left(f_b(x_t, a_t^*) - f_b(x_t, a_t) + 2 \zeta_t\right)\\
    &\le 2\zeta + \sum_{t=1}^T \left(f_b(x_t, a_t^b) - f_b(x_t, a_t)\right)\\
    &\le 2\zeta + \sum_{t=1}^T \left(f_t(x_t, a_t) - f_b(x_t,a_t)\right),
\$
where the first inequality uses Definition \ref{def:mis_bandit} and $\argmax_{a\in\cA}f_*(x_t,a)$ is denoted by $a_t^*$, the second inequality is due to $a_t^b=\argmax_{a\in\cA}f_b(x_t,a)$, and the last inequality is deduced from \eqref{eq:fb_xt_a}.
\end{proof}

\subsection{Step II: Sharper confidence Radius for Optimism}
\begin{lemma}\label{lm:Confidence_Radius_for_Optimismt_bandit}
We have $f_b\in\cF_t$ for all $t\in[T]$ with probability at least $1-\delta$ by taking 
$\lambda\le\beta^2$ and
$$
\beta_t=\tilde{c}_{\beta}\left(\eta\sqrt{\ln N(\gamma,\cF)/\delta} + \alpha\zeta + \gamma\sqrt{t} + \sqrt{\gamma\sqrt{tC_1(t,\zeta)}}\right),
$$
where $C_1(t,\zeta)=2(\zeta^2 + 2t\eta^2 + 3\eta^2\ln(2/\delta))$, and $\tilde{c}_{\beta}>0$ is an absolute constant.
\end{lemma}
\begin{proof}
By invoking Lemma \ref{lm:empirical_diff_mdp} with $\epsilon'=0$, we obtain that with probability at least $1-\delta$, for all $t\in[T]$:
\#\label{eq:sum_hat_ft-fb^2_sigma}
\sum_{s=1}^{t-1}(\hat{f}_t(z_s)-f_b(z_s))^2/\sigma_s^2 \le& 10\eta^2\ln(2N(\gamma,\cF)/\delta) + 5\sum_{s=1}^{t-1}|\hat{f_t}(z_s)-f_b(z_s)|\zeta_s/\sigma_s^2\nonumber\\
&\quad + 10\gamma(\gamma t + \sqrt{tC_1(t,\zeta)}),
\#
where $C_1(t,\zeta)=2(\zeta^2 + 2t\eta^2 + 3\eta^2\ln(2/\delta))$.

According to the definition of $\sigma_s$ in \eqref{eq:weight_bandit}, we have for all $s\le t-1$,
\#\label{eq:|hatft-fb|_sigma}
&|\hat{f_t}(z_s)-f_b(z_s)|/\sigma_s^2\nonumber\\
&\quad\le |\hat{f_t}(z_s)-\hat{f_s}(z_s))|/\sigma_s^2 + |f_b(z_s)-\hat{f_s}(z_s))|/\sigma_s^2\nonumber\\
&\quad\le \alpha\left(\sqrt{\lambda + \sum_{i=1}^{s-1}(\hat{f_t}(z_i)-\hat{f_s}(z_i))^2/\sigma_i^2} + \sqrt{\lambda + \sum_{i=1}^{s-1}(f_b(z_i)-\hat{f_s}(z_i))^2/\sigma_i^2}\right)\nonumber\\
&\quad\le 2\alpha\sqrt{\lambda+\beta_s^2},
\#
where the last inequality is due to $\hat{f_t}\in\cF_{t-1}\subseteq\cF_s$. Substituting \eqref{eq:|hatft-fb|_sigma} back into \eqref{eq:sum_hat_ft-fb^2_sigma}, we have
\$
&\left(\sum_{s=1}^{t-1}(\hat{f_t}(z_s)-f_b(z_s))^2/\sigma_s^2\right)^{1/2} \\
&\quad\le \left(10\eta^2\ln(2N(\gamma,\cF)/\delta)+10\alpha\zeta\sup_{s<t}\sqrt{\lambda+\beta_s^2}+10\gamma^2t+10\gamma\sqrt{tC_1(t,\zeta)}\right)^{1/2}\\
&\quad\le \eta\sqrt{10\ln(2N(\gamma,\cF)/\delta)} + \sqrt{10\alpha\zeta\sup_{s<t}\sqrt{\lambda+\beta_s^2}} + \gamma\sqrt{10t} + \sqrt{10\gamma\sqrt{tC_1(t,\zeta)}}\\
&\quad\le \eta\sqrt{10\ln(2N(\gamma,\cF)/\delta)} + 10\alpha\zeta + \sup_{s<t}\sqrt{\lambda+\beta_s^2}/4 + \gamma\sqrt{10t} + \sqrt{10\gamma\sqrt{tC_1(t,\zeta)}}\\
&\quad\le \beta,
\$

where the second inequality applies $\|x\|_2\le\|x\|_1$ for any vector $x$, the third inequality applies $\sqrt{2a\cdot b/2}\le a + b/4$, and the last inequality holds since $\lambda\le\beta^2$ and the upper bound of all $\beta_s$ is $\beta$.

\end{proof}

\subsection{Step III: Bound the Sum of Bonus}
We define event $E(T)=\{f_b\in\cF_t,~\forall~t\in[T]\}$. Lemma \eqref{lm:Confidence_Radius_for_Optimismt_bandit} shows that $\Pb(E(T))\ge1-\delta$. Therefore, by invoking Lemma~\ref{lm:Regret_under_Optimism_bandit}, we have with probability at least $1-\delta$,
\$
\reg(T) &\le 2\zeta + \sum_{t=1}^T \left(f_t(z_t) - f_b(z_t)\right)\notag\\
&\le 2\zeta + \sum_{t=1}^T\min\left(2,|f_t(z_t) - f_b(z_t)|\right)\nonumber\\
&= 2\zeta + \sum_{t=1}^T\min\left(2,|f_t(z_t) - \hat f_t(z_t)|\right) + \sum_{t=1}^T\min\left(2,|f_b(z_t) - \hat f_t(z_t)|\right).
\$
It follows that for any $f\in\cF_t$,
\#\label{eq:t_1T_reg}
\sum_{t=1}^T\min\left(2,|f(z_t) - \hat f_t(z_t)|\right)
&= \underbrace{\sum_{t:\sigma_t=1}\min\left(2,|f(z_t) - \hat f_t(z_t)|\right)}_{p_1} + \underbrace{\sum_{t:\sigma_t>1}\min\left(2,|f(z_t) - \hat f_t(z_t)|\right)}_{p_2}.
\#
Now, we bound these two terms as follows. For term $p_1$, we obtain
\#
p_1 &= \sum_{t:\sigma_t=1}\min\left(2,|f(z_t) - \hat f_t(z_t)|\right)\nonumber\\
&\le \sum_{t:\sigma_t=1}\min\left(2,\sqrt{2}\beta\cdot\frac{|f(z_t) - \hat f_t(z_t)|/\sigma_t}{\sqrt{\lambda+\sum_{i=1}^{t-1}\left(f(z_i)-\hat f_t(z_i)\right)^2/\sigma_i^2}}\right)\nonumber\\
&\le 2\beta\sum_{t:\sigma_t=1}\min\left(1,\frac{|f(z_t) - \hat f_t(z_t)|/\sigma_t}{\sqrt{\lambda+\sum_{i=1}^{t-1}\left(f(z_i)-\hat f_t(z_i)\right)^2/\sigma_i^2}}\right)\nonumber\\
&\le 2\beta\sum_{t=1}^TD_{\lambda,\sigma,\cF_t}(Z_1^t)\nonumber\\
&\le 2\beta\sqrt{T\sup_{Z_1^T} \sum_{t=1}^T (D_{\lambda,\sigma,\cF_t}(Z_1^t))^2},\label{eq:reg_p1}
\#
where we obtain the first inequality since $\lambda \leq \beta^2$ and
\#\label{eq:lambda_sum_ft_fb^2}
\lambda+\sum_{i=1}^{t-1}\left(f(z_i)-\hat f_t(z_i)\right)^2/\sigma_i^2 \le 2\beta^2,
\#
and the second inequality holds because $\beta\ge1$. Then, the term $p_2$ is bounded as follows:
\#
p_2 &= \sum_{t:\sigma_t>1}\min\left(2,|f(z_t) - \hat f_t(z_t)|\right)\nonumber\\
&= \sum_{t:\sigma_t>1}\min\left(2,\frac{|f(z_t) - \hat f_t(z_t)|}{\sqrt{\lambda+\sum_{s=1}^{t-1}\left(f(z_s)-\hat f_t(z_s)\right)^2/\sigma_s^2}}\cdot \sqrt{\lambda+\sum_{i=1}^{t-1}\left(f(z_i)-\hat f_t(z_i)\right)^2/\sigma_i^2}\right)\nonumber\\
&\le \sum_{t:\sigma_t>1}\min\left(2,\frac{|f(z_t) - \hat f_t(z_t)|/\sigma_t}{\sqrt{\lambda+\sum_{s=1}^{t-1}\left(f(z_s)-\hat f_t(z_s)\right)^2/\sigma_s^2}}\cdot\sqrt{2}\sigma_t\beta\right)\nonumber\\
&\le \sum_{t:\sigma_t>1}\min\left(2,\left(\sup_{f\in \cF_t}\frac{|f(z_t)-\hat{f_t}(z_t)|/\sigma_t}{\sqrt{\lambda + \sum_{s=1}^{t-1} (f(z_s)-\hat{f_t}(z_s))^2/\sigma_s^2}}\right)^2\cdot2\beta/\alpha\right)\nonumber\\
&\le 2\beta/\alpha\sum_{t=1}^T(D_{\lambda,\sigma,\cF_t}(Z_1^t))^2\nonumber\\
&\le 2\beta/\alpha\sup_{Z_1^T} \sum_{t=1}^T (D_{\lambda,\sigma,\cF_t}(Z_1^t))^2,\label{eq:reg_p2}
\#
where the first inequality invokes \eqref{eq:lambda_sum_ft_fb^2}, the second inequality is deduced since for $\sigma_t>1$,
\$
\sigma_t&= \frac{1}{\alpha}\sup_{f\in \cF_t} \frac{|f(z_t)-\hat{f_t}(z_t)|/\sigma_t}{\sqrt{\lambda + \sum_{s=1}^{t-1}\left(f(z_s)-\hat{f_t}(z_s)\right)^2/\sigma_s^2}},
\$
and the third inequality applies the definition of $D_{\lambda,\sigma,\cF_t}(Z_1^t)$ in \eqref{def:wdim_bandit}.
Let $\dim_{\sigma}$ denote $\sup_{Z_1^T} \sum_{t=1}^T (D_{\lambda,\sigma,\cF_t}(Z_1^t))^2$ for simplicity. Hence, by substituting the results in \eqref{eq:reg_p1} and \eqref{eq:reg_p2} back into \eqref{eq:t_1T_reg} with $f=f_t$ and $f_b$, we have with probability at least $1-\delta$,
\$
\reg(T) &\le 2\zeta + 4\beta\sqrt{T\dim_{\sigma}} + 4\beta/\alpha\dim_{\sigma}\\
&\le 4c_{\beta}\alpha\zeta\sqrt{T\dim_{\sigma}} + 4c_{\beta}\sqrt{T\dim_{\sigma}\ln(N(\gamma,\cF)/\delta)} + 4c_{\beta}\sqrt{c_0T\dim_{\sigma}}\\
&\qquad + 4\zeta\dim_{\sigma} + 4\dim_{\sigma}\sqrt{\ln(N(\gamma,\cF)/\delta)}/\alpha + 4\sqrt{c_0}\dim_{\sigma}/\alpha\\
&= \tilde{\cO}\left(\sqrt{T\dim_{\sigma}\ln N(\gamma,\cF)} + \zeta\dim_{\sigma}\right),
\$
where the last equality is obtained by setting $\alpha=\sqrt{\ln N(\gamma,\cF)}/\zeta$. 

Ultimately, invoking Lemma \ref{lm:Relationship_between_Eluder_Dimensions} results in the desired bound.

\section{Proof for the MDP Case}\label{sec:pf_th:mdp}
This section contains the proof of Theorem \ref{th:mdp}.

\subsection{Step I: Regret under Optimism}\label{ssec:step_I}
\begin{lemma}\label{lm:Regret_under_Optimism_mdp}
Assuming that $\cT_b^hf_t^{h+1}\in\cF_t^h$ (where $\cF_t^h$ is defined in Section~\ref{sec:mdp_alg}) for all $(t,h)\in[T]\times[H]$, we have
$$
\reg(T) \le H\zeta + \sum_{t=1}^T\sum_{h=1}^H\E_{\pi_t}\cE^h(f_t,x_t^h,a_t^h).
$$
\end{lemma}
\begin{proof}
We use the notation $f_t=\{f_t^h\}_{h=1}^{H+1}$ and 
$f_t^h(x) = \max_{a\in\cA}f_t^h(x,a)$ for any $x\in\cX$. Since $\cT_b^hf_t^{h+1}\in\cF_t^h$ for all $(t,h)\in[T]\times[H]$, we can invoke Lemma \ref{lm:optimism_mdp} to obtain that
\$
\reg(T) &= \sum_{t=1}^T[V_*^1(x_t^1)-V_{\pi_t}^1(x_t^1)]\\
&\le H\zeta + \sum_{t=1}^T[f_t^1(x_t^1)-V_{\pi_t}^1(x_t^1)]\\
&\le H\zeta + \sum_{t=1}^T\sum_{h=1}^H\E_{\pi_t}\cE^h(f_t,x_t^h,a_t^h),
\$
where the first inequality uses \eqref{eq:E_sum_V_*1}, the second inequality is obtained by Lemma \ref{lm:bellman_error}.
\end{proof}
\begin{lemma}\label{lm:optimism_mdp}
Assume that for all $(t,h)\in[T]\times[H]$, we have $\cT_b^hf_t^{h+1}\in\cF_t^h$. Then, the $f_t^h$ satisfies the following inequalities:
\#
&\sum_{t=1}^T V_*^1(x_t^1)\le \sum_{t=1}^T f_t^1(x_t^1) + H\zeta,\label{eq:E_sum_V_*1}\\
&\left|f_t^h(x_t^h,a_t^h) - (\cT^hf_t^{h+1})(x_t^h,a_t^h)\right|\le 2\beta_t^hb_t^h(x_t^h,a_t^h)+\zeta^h_t.\label{eq:E_fkh_Th_fkh+1}
\#
\end{lemma}
\begin{proof}
From $\cT_b^hf_t^{h+1}\in\cF_t^h$, we have
$$
\left(\sum_{s=1}^{t-1}\frac{\left(\hat{f}_t^h(x_s^h,a_s^h)-(\cT_b^h f_t^{h+1})(x_s^h,a_s^h)\right)^2}{(\sigma_s^h)^2} + \lambda\right)^{1/2}\le \beta_t^h.
$$
Combining the inequality above and the definition of $b_t^h(\cdot)$ in \eqref{eq:bonus_mdp}, we obtain
$$
\left|\hat{f}_t^h(x^h,a^h)-(\cT_b^h f_t^{h+1})(x^h,a^h)\right| \le \beta_t^hb_t^h(x^h,a^h).
$$
Hence, by taking $g_t^{h+1}=f_t^{h+1}$ for all $(t,h)\in[T]\times[H]$ in Definition \ref{def:mis_mdp}, we have $|(\cT^hf_t^{h+1}-\cT_b^hf_t^{h+1})(x_t^h,a_t^h)|\le\zeta^h_t$. It follows that
\#\label{eq:second_desired_bound}
-\zeta^h_t \le f_t^h(x_t^h,a_t^h)-(\cT^h f_t^{h+1})(x_t^h,a_t^h) \le 2\beta_t^hb_t^h(x_t^h,a_t^h)+\zeta_t^h,
\#
which validate \eqref{eq:E_fkh_Th_fkh+1}. Then, we will demonstrate \eqref{eq:E_sum_V_*1} by induction. When $h=H+1$, we know that $Q_*^{H+1}=f_t^{H+1}=0$. Assume that at step $h+1$, we have 
$$
\sum_{t=1}^TQ_*^{h+1}(x_t^{h+1},a_t^{h+1})\le \sum_{t=1}^Tf_t^{h+1}(x_t^{h+1},a_t^{h+1})+(H-h)\zeta.
$$
Then, at step $h$, we obtain that
\$
\sum_{t=1}^T\left(Q_*^h(x_t^h,a_t^h)-f_t^h(x_t^h,a_t^h)\right) &\le \sum_{t=1}^T \left((\cT^hQ_*^{h+1})(x_t^h,a_t^h)-(\cT^hf_t^{h+1})(x_t^h,a_t^h) + \zeta_t^h\right)\\
&\le \sum_{t=1}^T\E\left[V_*^{h+1}(x_t^{h+1})-f_t^{h+1}(x_t^{h+1})\right] + \zeta\\
&\le (H-h+1)\zeta,
\$
where the first inequality uses \eqref{eq:second_desired_bound}, the second inequality invokes the definition of Bellman operator and cumulative corruption, and the last inequality is due to the induction hypothesis at step $h+1$. 

Then, it follows that
\$
\sum_{t=1}^T V_*^h(x_t^h) &= \sum_{t=1}^T Q_*^h(x_t^h, \pi_*^h(x_t^h))\\
&\le \sum_{t=1}^T f_t^h(x_t^h,\pi_*^h(x_t^h)) + (H-h)\zeta\\
&\le  \sum_{t=1}^T f_t^h(x_t^h) + (H-h)\zeta,
\$
where the last inequality is because of the greedy step for the optimistic function $f_t^h$ in Algorithm \ref{alg:ucb_nonlinear_mdp}.
\end{proof}

\subsection{Step II: Confidence Radius for Optimism}\label{ssec:step_II}
In this step, we derive the confidence radius for $f_t^h$ in CR-LSVI-UCB (Algorithm \ref{alg:ucb_nonlinear_mdp}). 
\begin{lemma}\label{lm:ucb_mdp}
In Algorithm \ref{alg:ucb_nonlinear_mdp} under Assumption \ref{as:mdp}, for all $(t,h)\in[T]\times[H]$, we have $\cT_b^hf_t^{h+1}\in\cF_t^h$ with probability at least $1-\delta$, where we set for each $t>0$ that $\beta_t^{H+1}=0$ and from $h=H$ to $h=1$,
\$
(\beta_t^h)^2 \ge& 12\lambda + 12\ln(2HN_T^h(\gamma)/\delta) + 12\alpha\zeta\sup_{s< t}\beta_{s-1}^h + 12(5\sup_s\beta_s^{h+1}\gamma)^2T\\
&\quad~~  + 60\sup_s\beta_s^{h+1}\gamma\sqrt{TC_1(t,\zeta)},
\$
where $N_T^h(\gamma)=N(\gamma,\cF^h)\cdot N(\gamma,\cF^{h+1})\cdot N(\gamma,\cB^{h+1}(\lambda))$ and $C_1(t,\zeta)=2(\zeta^2 + 2t + 3\ln(2/\delta))$.
\end{lemma}
\begin{proof}
We make the statement that with probability at least $1-\delta$, we have for all $(t,h)\in[T]\times[H]$ and any $\tau\ge t$,
\#\label{eq:sum_hatf_th-Tbf_th+1}
\left(\sum_{s=1}^{t-1}\frac{\left(\hat{f}_t^h(x_s^h,a_s^h)-(\cT_b^h f_{\tau}^{h+1})(x_s^h,a_s^h)\right)^2}{(\sigma_s^h)^2} + \lambda\right)^{1/2}\le \beta_t^h,
\#
which will be proved by induction. Firstly, the statement holds for $t=1$.

Then, for $t>1$, suppose that \eqref{eq:sum_hatf_th-Tbf_th+1} holds for all $s\le t-1$, which means that for all $(s,h)\in[t-1]\times[H]$ and any $\tau\ge s$,
\$
\left(\sum_{i=1}^{s-1}\frac{\left(\hat{f}_s^h(x_i^h,a_i^h)-(\cT_b^h f_{\tau}^{h+1})(x_i^h,a_i^h)\right)^2}{(\sigma_i^h)^2} + \lambda\right)^{1/2}\le \beta_s^h.
\$

Now, for round $t$, fix $h\in[H]$ and $\tau\ge t$. Define $\cF_{\gamma}^{h+1}$ as a $\gamma~\|\cdot\|_{\infty}$ cover of $\cF^{h+1}$, and $\cB^{h+1}_{\gamma}$ as an $\gamma~\|\cdot\|_{\infty}$ cover of $\cB^{h+1}(\lambda)$. Then, we construct $\Bar{\cF}_{\gamma}^{h+1}=\cF_{\gamma}^{h+1}\oplus\beta_{\tau}^{h+1}\cB^{h+1}_{\gamma}$ as a $(1+\beta_{\tau}^{h+1})\gamma~\|\cdot\|_{\infty}$ cover of $\{f_{\tau}^{h+1}(\cdot)\}$. Thus, given $f_{\tau}^{h+1}$, let $\Bar{f}_{\tau}^{h+1}\in\Bar{\cF}_{\gamma}^{h+1}$ so that $\|\Bar{f}_{\tau}^{h+1}-f_{\tau}^{h+1}\|_{\infty}\le\beps=(1+\beta_{\tau}^{h+1})\gamma$. Define $\Bar{y}_s^h=r_s^h+\Bar{f}_{\tau}^{h+1}(x_s^{h+1})$ and
$$
\tilde{f}_t^h = \argmin_{f^h\in\cF_{t-1}^h}\sum_{s=1}^{t-1}(f^h(x_s^h,a_s^h)-\Bar{y}_s^h)^2.
$$
Then, we find that 
\#\label{eq:approximate_erm}
&\left(\sum_{s=1}^{t-1}(\hat{f}_t^h(x_s^h,a_s^h)-\Bar{y}_s^h)^2\right)^{1/2} \le \left(\sum_{s=1}^{t-1}(\hat{f}_t^h(x_s^h,a_s^h)-y_s^h)^2\right)^{1/2}+\sqrt{t}\beps\nonumber\\
&\qquad \le \left(\sum_{s=1}^{t-1}(\tilde{f}_t^h(x_s^h,a_s^h)-y_s^h)^2\right)^{1/2}+\sqrt{t}\beps \le \left(\sum_{s=1}^{t-1}(\tilde{f}_t^h(x_s^h,a_s^h)-\Bar{y}_s^h)^2\right)^{1/2}+2\sqrt{t}\beps,
\#
where the first and third inequality is obtained by applying $\|\Bar{f}_{\tau}^{h+1}-f_{\tau}^{h+1}\|_{\infty}\le\beps$, and the second inequality uses the fact that $\hat{f}_t^h$ is the ERM solution of the least squares problem. Therefore, we replace $f_*$ by $\E[\Bar{y}_s^h|x_s^h,a_s^h]=(\cT^h\Bar{f}_{\tau}^{h+1})(x_s^h,a_s^h)$ and $f_b$ by $\cT_b^h\Bar{f}_{\tau}^{h+1}$ in Lemma \ref{lm:empirical_diff_mdp} so that from the corruption in Definition \ref{def:mis_mdp} for corruption,
\$
|(f_b-f_*)(x_s^h,a_s^h)| \le |(\cT_b^h\bar f_{\tau}^{h+1}-\cT^h\bar f_{\tau}^{h+1})(x_s^h,a_s^h)| \le \zeta_s^h.
\$
Then, since $\hat f_t^h$ is the approximate ERM in \eqref{eq:approximate_erm}, we can invoke Lemma \ref{lm:empirical_diff_mdp} with the $f_*$ and $f_b$, $\epsilon'=2\beps$, $\eta=1$ and $\zeta_s=\zeta_s^h$. 
Then, by taking a union bound over $\Bar{f}_{\tau}^{h+1}\in\Bar{\cF}_{\gamma}^{h+1}$ and $h\in[H]$, we can verify that with probability at least $1-\delta$, the following inequality holds for all $t\in[T]$:
\#\label{eq:sum_hat_fkh_Tb_barf0}
&\sum_{s=1}^{t-1}\frac{\left(\hat{f}_t^h(x_s^h,a_s^h)-(\cT_b^h\Bar{f}_{\tau}^{h+1})(x_s^h,a_s^h)\right)^2}{(\sigma_s^h)^2}\nonumber\\
&\quad\le 10\ln(2HN_T^h(\gamma)/\delta) + 5\sum_{s=1}^{t-1}|\hat{f}_t^h(x_s^h,a_s^h)-(\cT_b^h\Bar{f}_{\tau}^{h+1})(x_s^h,a_s^h)|\cdot\zeta_s^h/(\sigma_s^h)^2\nonumber\\
&\qquad + 10(\gamma + 2\beps)\cdot\left((\gamma + 2\beps)t + \sqrt{tC_1(t,\zeta)}\right),
\#
where $C_1(t,\zeta)=2(\zeta^2 + 2t + 3\ln(2/\delta))$. Further, for all $s\le t-1$, the weights' definition in \eqref{eq:weight_mdp} indicates that
\$
&|\hat{f}_t^h(x_s^h,a_s^h)-(\cT_b^h\Bar{f}_{\tau}^{h+1})(x_s^h,a_s^h)|/(\sigma_s^h)^2\\
&\quad\le |\hat{f}_t^h(x_s^h,a_s^h)-(\cT_b^hf_{\tau}^{h+1})(x_s^h,a_s^h)|/(\sigma_s^h)^2 + \gamma\\
&\quad\le |\hat{f}_t^h(x_s^h,a_s^h)-\hat{f}_s^h(x_s^h,a_s^h)|/(\sigma_s^h)^2 +  |(\cT_b^hf_{\tau}^{h+1})(x_s^h,a_s^h)-\hat{f}_s^h(x_s^h,a_s^h)|/(\sigma_s^h)^2 + \gamma\\
&\quad\le 2\alpha\beta_s^h + \gamma,
\$
where the last inequality is due to $\hat{f}_t^h\in\cF_{t-1}^h\subset\cF_s^h$ and the induction hypothesis that $\cT_b^hf_{\tau}^{h+1}\in\cF_s^h$ for $\tau\ge s$. Then, combining the inequality above and \eqref{eq:sum_hat_fkh_Tb_barf0}, we deduce that
\#\label{eq:sum_hat_fkh_Tb_barf}
&\sum_{s=1}^{t-1}\frac{\left(\hat{f}_t^h(x_s^h,a_s^h)-(\cT_b^h\Bar{f}_{\tau}^{h+1})(x_s^h,a_s^h)\right)^2}{(\sigma_s^h)^2}\nonumber\\
&\quad\le 10\ln(2HN_T^h(\gamma)/\delta) + 10\alpha\zeta\sup_{s< t}\beta_s^h + 5\gamma\zeta + 10(\gamma + 2\beps)\cdot((\gamma + 2\beps)t + \sqrt{tC_1(t,\zeta)}).
\#
Therefore, it follows that with probability at least $1-\delta$,
\$
&\left(\sum_{s=1}^{t-1}\frac{\left(\hat{f}_t^h(x_s^h,a_s^h)-(\cT_b^h f_{\tau}^{h+1})(x_s^h,a_s^h)\right)^2}{(\sigma_s^h)^2} + \lambda\right)^{1/2}\\
&\quad\le \left(\sum_{s=1}^{t-1}\frac{\left(\hat{f}_t^h(x_s^h,a_s^h)-(\cT_b^h\Bar{f}_{\tau}^{h+1})(x_s^h,a_s^h)\right)^2}{(\sigma_s^h)^2}\right)^{1/2} + \sqrt{t}\beps + \sqrt{\lambda}\\
&\quad\le \left(10\ln(2HN_T^h(\gamma)/\delta) + 10\alpha\zeta\sup_{s< t}\beta_s^h + 5\gamma\zeta + 10(2\beta_{\tau}^{h+1}+3)^2\gamma^2T\right.\nonumber\\
&\qquad \left. + 10(2\beta_{\tau}^{h+1}+3)\gamma\sqrt{TC_1(t,\zeta)}\right)^{1/2} + (\beta_{\tau}^{h+1}+1)\gamma\sqrt{T} + \sqrt{\lambda}\\
&\quad\le \left(12\lambda + 12\ln(2HN_T^h(\gamma)/\delta) + 12\alpha\zeta\sup_{s< t}\beta_{s-1}^h + 12(5\sup_s\beta_s^{h+1}\gamma)^2T\right.\\
&\qquad \left. + 60\sup_s\beta_s^{h+1}\gamma\sqrt{TC_1(t,\zeta)}\right)^{1/2} \le \beta_t^h,
\$
where the first inequality uses the triangle inequality, the second inequality applies \eqref{eq:sum_hat_fkh_Tb_barf}, and the second last inequality uses Cauchy-Schwarz inequality. Therefore, we validate the statement in \eqref{eq:sum_hatf_th-Tbf_th+1}. For all $(t,h)\in[T]\times[H]$, by taking $\tau=t$ in \eqref{eq:sum_hatf_th-Tbf_th+1}, we finally complete the proof.
\end{proof}

Then, we demonstrate the following lemma with techniques similar to Lemma \ref{lm:Confidence_Radius_for_Optimismt_bandit}.

\subsection{Step III: Bound the Sum of Bonuses}\label{ssec:step_III}
Ultimately, we bound the cumulative regret for MDPs with corruption in Theorem \ref{th:mdp} as follows.

Recall the definition of event $E(T)$ in Lemma \ref{lm:Regret_under_Optimism_mdp}: for all $(t,h)\in[T]\times[H]$, $f_t^h$ satisfies the following inequalities:
\$
&\sum_{t=1}^T V_*^1(x_t^1)\le \sum_{t=1}^T f_t^1(x_t^1) + H\zeta,\\
&\left|f_t^h(x_t^h,a_t^h) - (\cT^hf_t^{h+1})(x_t^h,a_t^h)\right|\le 2\beta_t^hb_t^h(x_t^h,a_t^h)+\zeta^h_t.
\$
Since taking $\beta_t^h=\beta$ for all $(t,h)\in[T]\times[H]$ satisfies Lemma \ref{lm:ucb_mdp}, event $E(T)$ holds with probability at least $1-\delta$. Therefore, assuming $E(T)$ happens, we obtain by applying Lemma \ref{lm:Regret_under_Optimism_mdp} that
\#\label{eq:reg_T_bandit}
\reg(T) &\le H\zeta + \sum_{t=1}^T\sum_{h=1}^H\E_{\pi_t}\cE^h(f_t,x_t^h,a_t^h)\nonumber\\
&\le 2H\zeta + 2\underbrace{\sum_{(t,h):\sigma_t^h=1}\E_{\pi_t}\min(1,\beta_t^hb_t^h(x_t^h,a_t^h))}_{p_1}\nonumber\\
&\qquad + 2\underbrace{\sum_{(t,h):\sigma_t^h>1}\E_{\pi_t}\min(1,\beta_t^hb_t^h(x_t^h,a_t^h))}_{p_2},
\#
where the second inequality applies \eqref{eq:E_fkh_Th_fkh+1} and $(f_t^h - \cT_b^hf_t^{h+1})(x_t^h,a_t^h)\le 2$. Then, we bound the two terms above respectively. For the first term, we deduce that
\$
p_1 &\le \sum_{(t,h):\sigma_t^h=1}\E_{\pi_t}\max(1,\beta_t^h)\cdot\min(1,b_t^h(x_t^h,a_t^h))\\
&\le \sqrt{\sum_{t=1}^T\sum_{h=1}^H\max(1,(\beta_t^h)^2)}\cdot \E_{\pi_t}\sqrt{ \sum_{(t,h):\sigma_t^h=1}\min(1,(b_t^h(x_t^h,a_t^h))^2)}\\
&\le \sqrt{TH}(1+\beta)\sqrt{\sum_{h=1}^H\sup_{Z_h^T}\sum_{t=1}^T (D_{\lambda,\sigma^h,\cF_t^h}(Z_h^t))^2},
\$
where the first inequality is due to the fact that $\text{min}(a_1a_2,b_1b_2)\le\text{max}(a_1,b_1)\cdot\text{min}(a_2,b_2)$, the second inequality is obtained by using Cauchy-Schwarz inequality, and the last inequality utilizes the definition of $D_{\lambda,\sigma^h,\cF_t^h}(Z_h^t)$ in \eqref{def:wdim_mdp} and the selection of confidence radius: $\beta_t^h=\beta$. 

Then, for $\sigma_t^h>1$, according to the definition of $\sigma_t^h$ in \eqref{eq:weight_mdp}, we have $(\sigma_t^h)^2=1/\alpha\cdot b_t^h(x_t^h,a_t^h)$. Thus, we can bound the second term as
\$
p_2 &\le \sum_{(t,h):\sigma_t^h>1}\E_{\pi_t}\min(1,\beta_t^h(\sigma_t^h)^2 \cdot b_t^h(x_t^h,a_t^h)/(\sigma_t^h)^2)\\
&\le \sum_{(t,h):\sigma_t^h>1}\E_{\pi_t}\min(1,\beta_t^h/\alpha\cdot (b_t^h(x_t^h,a_t^h))^2/(\sigma_t^h)^2)\\
&\le \beta/\alpha\cdot\sum_{t=1}^T\sum_{h=1}^H\E_{\pi_t}\min(1,(b_t^h(x_t^h,a_t^h))^2/(\sigma_t^h)^2)\\
&\le \beta/\alpha\cdot \sum_{h=1}^H \sum_{t=1}^T \E_{\pi_t}(D_{\lambda,\sigma^h,\cF_t^h}(Z_h^t))^2\\
&\le \beta/\alpha\cdot\sum_{h=1}^H\sup_{Z_h^T}\sum_{t=1}^T (D_{\lambda,\sigma^h,\cF_t^h}(Z_h^t))^2,
\$ 
where the $D_{\lambda,\sigma^h,\cF_t^h}(Z_h^t)$ is formulated in Definition \ref{def:wdim_mdp}.
Combining these results, we get
\$
\reg(T) &\le 2H\zeta + \sqrt{TH}(1+\beta)\sqrt{\sum_{h=1}^H\sup_{Z_h^T}\sum_{t=1}^T (D_{\lambda,\sigma^h,\cF_t^h}(Z_h^t))^2} + \beta/\alpha\cdot\sum_{h=1}^H\sup_{Z_h^T}\sum_{t=1}^T (D_{\lambda,\sigma^h,\cF_t^h}(Z_h^t))^2\\
&= \Tilde{\cO}\left(\left(H+\sum_{h=1}^H\sup_{Z_h^T}\sum_{t=1}^T (D_{\lambda,\sigma^h,\cF_t^h}(Z_h^t))^2\right)\zeta + \sqrt{TH\ln(N_T(\gamma))\sum_{h=1}^H\sup_{Z_h^T}\sum_{t=1}^T (D_{\lambda,\sigma^h,\cF_t^h}(Z_h^t))^2}\right.\\
&\qquad\left. + \alpha\zeta\sqrt{TH\sum_{h=1}^H\sup_{Z_h^T}\sum_{t=1}^T (D_{\lambda,\sigma^h,\cF_t^h}(Z_h^t))^2} + \sqrt{\ln(N_T(\gamma))}\sum_{h=1}^H\sup_{Z_h^T}\sum_{t=1}^T (D_{\lambda,\sigma^h,\cF_t^h}(Z_h^t))^2)/\alpha\right)\\
&=\Tilde{\cO}\left(\sqrt{TH\ln(N_T(\gamma))\sum_{h=1}^H\sup_{Z_h^T}\sum_{t=1}^T (D_{\lambda,\sigma^h,\cF_t^h}(Z_h^t))^2} + \zeta\sum_{h=1}^H\sup_{Z_h^T}\sum_{t=1}^T (D_{\lambda,\sigma^h,\cF_t^h}(Z_h^t))^2\right),
\$
where the first inequality is deduced by taking the bounds of terms $p_1$ and $p_2$ back into \eqref{eq:reg_T_bandit}, the first equality uses the choice of $\beta=\cO(\alpha\zeta + \sqrt{\ln(H\ln(N_T(\gamma))/\delta)})$, and the last equation is obtained by setting $\alpha=\sqrt{\ln(N_T(\gamma))}/\zeta$.

Then, it suffices to replace weighted eluder dimension $\sup_{Z_h^T}\sum_{t=1}^T (D_{\lambda,\sigma^h,\cF_t^h}(Z_h^t))^2$ with the eluder dimension $\dim_E(\cF,\epsilon)$ in Definition \ref{def:eluder_dim}. Because $\cF$ is factorized as $\prod_{h=1}^H\cF^h$, we get
$$
\dim_E(\cF,\epsilon) = \sum_{h=1}^H \dim_E(\cF^h,\epsilon).
$$
By invoking Lemma \ref{lm:Relationship_between_Eluder_Dimensions} for each function space $\cF^h$, we obtain
$$
\sup_{Z_h^T}\sum_{t=1}^T (D_{\lambda,\sigma^h,\cF_t^h}(Z_h^t))^2 \le (\sqrt{8c_0}+3)\dim_E(\cF^h,\lambda/T)\log(T/\lambda)\ln T,
$$
which indicates that
$$
\sum_{h=1}^H\sup_{Z_h^T}\sum_{t=1}^T (D_{\lambda,\sigma^h,\cF_t^h}(Z_h^t))^2 \le (\sqrt{8c_0}+3)\dim_E(\cF,\lambda/T)\log(T/\lambda)\ln T.
$$
Therefore, it follows that
$$
\reg(T) = \Tilde{\cO}\left(\sqrt{TH\ln(N_T(\gamma))\dim_E(\cF,\lambda/T)} + \zeta\left(H+\dim_E(\cF,\lambda/T)\right)\right).
$$

\section{Relationship between Confidence and Eluder Dimension}\label{s:Relationship_between_Confidence_and_Eluder_Dimension}
For simplicity, given a data sequence $\cZ=\{z_t\}_{t\in[T]}$, a positive weight sequence $\{\sigma_t\}_{t\in[T]}$ and functions $f,f'$ defined on a space containing $\cZ$, we denote the first $t$ elements in $\cZ$ by $\cZ_t=\{z_s\}_{s\in[t]}$ and define the quadratic error and its weighted version as follows: 
\$
&\|f-f'\|_{\cZ} = \left(\sum_{z_i\in\cZ} (f(z_i)-f'(z_i))^2\right)^{1/2},\\
&\|f-f'\|_{\cZ,\sigma} = \left(\sum_{z_i\in\cZ} \frac{(f(z_i)-f'(z_i))^2}{\sigma_i^2}\right)^{1/2}.
\$

The following lemma reveals the relationship between the sum of eluder-like confidence quantity and the eluder dimension.

\begin{proof}[Proof of Lemma \ref{lm:Relationship_between_Eluder_Dimensions}]
Let $\cZ=\{z_t\}_{t=1}^T\subset\cX\times\cA$. 

\textbf{Step I: Matched levels.}
To begin with, we divide the sequence $\cZ$ into $\log(T/\lambda)+1$ disjoint sets. For each $z_t\in\cZ$, Let $f\in\cF_t$ be the function that maximizes
$$
\frac{(f(z_t)-\hat{f_t}(z_t))^2/\sigma_t^2}{\lambda+\|f-\hat{f_t}\|_{\cZ_{t-1},\sigma}^2},
$$
where we define $\cZ_{t-1}=\{z_s\}_{s\in[t-1]}$ and $L(z)=(f(z)-\hat{f_t}(z))^2$ for such $f$. Since $L(z)\in[0,1]$, we decompose $\cZ$ into $\log(T/\lambda) + 1$ disjoint subsequences:
$$
\cZ = \cup_{\iota=0}^{\log(T/\lambda)} \cZ^{\iota},
$$
where we define for $\iota=0,\ldots,\log(T/\lambda)-1$,
$$
\cZ^{\iota} = \{z_t\in\cZ \,|\, L(z_t)\in(2^{-\iota-1},2^{-\iota}]\},
$$
and
$$
\cZ^{\log(T/\lambda)} = \{z_t\in\cZ \,|\, L(z_t)\in[0,\lambda/T]\}.
$$
Correspondingly, we can also divide $\rR^+$ into $\log(T/\lambda)+2$ disjoint subsets:
$$
\rR^+ = \cup_{\iota=-1}^{\log(T/\lambda)} \cR^{\iota},
$$
where we define $\cR^{\iota}=[2^{\iota/2}\ln N, 2^{(\iota+1)/2}\ln N)$ for $\iota=0,\ldots,\log(T/\lambda)-1$, $\cR^{\log(T/\lambda)}=[\sqrt{T/\lambda}\ln N,+\infty)$, and $\cR^{-1} = [0, \ln N)$. Since $\zeta\in\rR^+$, there exists an $\iota_0\in\{-1,0,\ldots,\log(T/\lambda)\}$ such that $\zeta\in\cR^{\iota_0}$. 

\textbf{Step II: Control weights in each level.}
Now, for any $z_t\in\cZ^{\log(T/\lambda)}$, we have
$$
\sup_{f\in\cF_t} \frac{(f(z_t)-\hat{f_t}(z_t))^2/\sigma_t^2}{\lambda+\|f-\hat{f_t}\|_{\cZ_{t-1},\sigma}^2} \le \frac{L(z_t)}{\lambda} \le \frac{1}{T},
$$
which implies that
\#\label{eq:D_lambda_F_t_1}
\sum_{z_t\in\cZ^{\log(T/\lambda)}} (D_{\lambda,\sigma,\cF}(Z_1^t))^2 \le 1.
\#

Moreover, for $0\le\iota\le\log(T/\lambda)-1$, we need to control the upper and lower bound of weights $\{\sigma_t: z_t\in\cZ^{\iota}\}$. For convenience, define for $t\in[T]$,
$$
\psi_t = \frac{1}{\alpha}\sup_{f\in\cF_t} \frac{|f(z_t)-\hat{f_t}(z_t)|}{\sqrt{\lambda+\|f-\hat{f_t}\|_{\cZ_{t-1},\sigma}^2}}.
$$
Besides, we observe that
$$
\argmax_{f\in\cF_t} \frac{(f(z_t)-\hat{f_t}(z_t))^2/\sigma_t^2}{\lambda+\|f-\hat{f_t}\|_{\cZ_{t-1},\sigma}^2} = \argmax_{f\in\cF_{t}} \frac{|f(z_t)-\hat{f_t}(z_t)|}{\sqrt{\lambda+\|f-\hat{f_t}\|_{\cZ_{t-1},\sigma}^2}},
$$
so we can take the $f$ as the solution of the two terms above for $\psi_t$. Hence, we have $|f(z_t)-\hat{f_t}(z_t)|=\sqrt{L(z_t)}$.

Now, consider two situations for $\iota$. When $\iota>\iota_0$, we have $\zeta<2^{(\iota_0+1)/2}\ln N \le 2^{\iota/2}\ln N$. For any $z_t\in\cZ^{\iota}$, we know that $L(z_t)\le 2^{-\iota}$. Hence, it follows that
$$
\psi_t \le \frac{1}{\alpha}\cdot\frac{2^{-\iota/2}}{\sqrt{\lambda}} \le \frac{\zeta}{\sqrt{\ln N}}\cdot\frac{2^{-\iota/2}}{\sqrt{\ln N}} \le 1,
$$
where the second inequality is deduced by taking $\alpha=\sqrt{\ln N}/\zeta$ and $\lambda=\ln N$. Therefore, since $\sigma_t^2=\max(1,\psi_t)$, we obtain $\sigma_t=1$ for all $z_t\in\cZ^{\iota}$.

When $\iota\le\iota_0$, we have $\zeta\ge 2^{\iota_0/2}\ln N\ge 2^{\iota/2}\ln N$, and $2^{-\iota-1}\le L(z_t)\le 2^{-\iota}$ for all $z_t\in\cZ^{\iota}$. Then, we can verify that
\$
&\psi_t \le \frac{\zeta}{\sqrt{
\ln N}}\cdot\frac{2^{-\iota/2}}{\sqrt{\ln N}} = \frac{\zeta}{2^{\iota/2}\ln N},\\
&\psi_t \ge \frac{\zeta}{\sqrt{
\ln N}}\cdot\frac{2^{-(\iota+1)/2}}{\sqrt{c_0\ln N}} = \frac{\zeta}{\sqrt{2c_0}2^{\iota/2}\ln N},
\$
where the inequality of the second row applies
$$
\lambda + \|f-\hat{f_t}\|_{\cZ_{t-1},\sigma}^2 \le \lambda + \beta_t^2 \le c_0 \ln N,
$$
where $c_0> 1$ is an absolute constant. Then, since $\zeta/(2^{\iota/2}\ln N)\ge 1$, we know that $\zeta/(\sqrt{2c_0}2^{\iota/2}\ln N)\le \max(1,\psi_t) \le \zeta/(2^{\iota/2}\ln N)$. Therefore, we get for all $z_t\in\cZ^{\iota}$,
$$
\sigma_t^2\in[\zeta/(\sqrt{2c_0}2^{\iota/2}\ln N), \zeta/(2^{\iota/2}\ln N)].
$$

\textbf{Step III: Bound the sum.}
Now, we bound $\sum_{z_t\in\cZ^{\iota}} (D_{\lambda,\sigma,\cF}(Z_1^t))^2$ for $\iota\in[0,\log(T/\lambda)-1]$. For each $\iota=0,\ldots,\log(T/\lambda)-1$, we decompose $\cZ^{\iota}$ into $N^{\iota}+1$ disjoint subsets: $\cZ^{\iota} = \cup_{j=1}^{N^{\iota}+1} \cZ_j^{\iota}$, where we define $N^{\iota} = |\cZ^{\iota}|/\dim_E(\cF,2^{(-\iota-1)/2})$. With a slight abuse of notation, let $\cZ^{\iota}=\{z_i\}_{i\in[|\cZ^{\iota}|]}$, where the elements are arranged in the same order as the original set $\cZ$. Initially, let $\cZ_j^{\iota}=\{\}$ for all $j\in[N^{\iota}+1]$. Then, from $i=1$ to $|\cZ^{\iota}|$, we find the smallest $j\in[N^{\iota}]$ such that $z_i$ is $2^{(-\iota-1)/2}$-independent of $\cZ_j^{\iota}$ with respect to $\cF$. If such a $j$ does not exist, set $j=N^{\iota}+1$. Then, we denote the choice of $j$ for each $z_i$ by $j(z_i)$. According to the design of the procedure, it is obvious that for all $z_i\in\cZ^{\iota}$, $z_i$ is $2^{(-\iota-1)/2}$-dependent on each of $\cZ_{1,i}^{\iota},\ldots,\cZ_{j(z_i)-1,i}^{\iota}$, where $\cZ_{k,i}^{\iota}=\cZ_{k}^{\iota}\cap\{z_1,\ldots,z_{i-1}\}$ for $k=1,\ldots,j(z_i)-1$.

For any $z_i\in\cZ^{\iota}$ indexed by $t$ in $\cZ$, by taking the function $f$ that maximizes
$$
\frac{(f(z_t)-\hat{f_t}(z_t))^2/\sigma_t^2}{\lambda+\|f-\hat{f_t}\|_{\cZ_{t-1},\sigma}^2},
$$
we have $(f(z_i)-\hat{f_t}(z_i))^2\ge 2^{-\iota-1}$. Additionally, because $z_i$ is $2^{(-\iota-1)/2}$-dependent on each of $\cZ_{1,i}^{\iota},\ldots,\cZ_{j(z_i)-1,i}^{\iota}$, we get for each $k=1,\ldots,j(z_i)-1$,
$$
\|f-\hat{f_t}\|_{\cZ_{k,i}^{\iota}}^2 \ge 2^{-\iota-1}.
$$
It follows that
$$
\frac{(f(z_t)-\hat{f_t}(z_t))^2/\sigma_t^2}{\lambda+\|f-\hat{f_t}\|_{\cZ_{t-1},\sigma}^2} \le \frac{2^{-\iota}/\sigma_t^2}{\lambda+\sum_{k=1}^{j(z_i)-1}\|f-\hat{f_t}\|_{\cZ_{k,i}^{\iota},\sigma}^2}.
$$

When $\iota>\iota_0$, we have $\sigma_t=1$ for all $z_t\in\cZ^{\iota}$. Thus, it follows that
$$
\frac{(f(z_t)-\hat{f_t}(z_t))^2/\sigma_t^2}{\lambda+\|f-\hat{f_t}\|_{\cZ_{t-1},\sigma}^2} \le \frac{2^{-\iota}}{\lambda + (j(z_i)-1)2^{-\iota-1}} = \frac{2}{j(z_i)-1 + \lambda2^{\iota+1}}.
$$
Summing over all $z_t\in\cZ^{\iota}$, we obtain
\#\label{eq:D_lambda_F_t_2}
\sum_{z_t\in\cZ^{\iota}} (D_{\lambda,\sigma,\cF}(Z_1^t))^2 &\le \sum_{j=1}^{N^{\iota}}\sum_{z_i\in\cZ_j^{\iota}} \frac{2}{j-1 + \lambda2^{\iota+1}} + \sum_{z_i\in\cZ_{N^{\iota}+1}^{\iota}} \frac{2}{N^{\iota}}\nonumber\\
&\le \sum_{j=1}^{N^{\iota}} \frac{2|\cZ_j^{\iota}|}{j} + \frac{2|\cZ_{N^{\iota}+1}^{\iota}|}{N^{\iota}}\nonumber\\
&\le 2\dim_E(\cF,2^{(-\iota-1)/2})\ln N^{\iota} + 2|\cZ^{\iota}|\cdot\frac{\dim_E(\cF,2^{(-\iota-1)/2})}{|\cZ^{\iota}|}\nonumber\\
&\le 4\dim_E(\cF,2^{(-\iota-1)/2})\ln N^{\iota},
\#
where the third inequality is deduced since by the definition of eluder dimension, we have $|\cZ_j^{\iota}|\le\dim_E(\cF,2^{(-\iota-1)/2})$ for all $j\in[N^{\iota}]$.

When $\iota\le\iota_0$, we have $\sigma_s^2\in[\zeta/(\sqrt{2c_0}2^{\iota/2}\ln N), \zeta/(2^{\iota/2}\ln N)]$ for all $z_s\in\cZ^{\iota}$. Therefore, their weights are roughly of the same order. Then, we can verify that
\$
\frac{(f(z_t)-\hat{f_t}(z_t))^2/\sigma_t^2}{\lambda+\|f-\hat{f_t}\|_{\cZ_{t-1},\sigma}^2} &\le \frac{(f(z_t)-\hat{f_t}(z_t))^2/\sigma_t^2}{\lambda+\|f-\hat{f_t}\|_{\cZ_{t-1}\cap\cZ^{\iota},\sigma}^2}\\
&\le \frac{2^{-\iota}\sqrt{2c_0}2^{\iota/2}\ln N/\zeta}{\lambda+(j(z_i)-1)2^{-\iota-1}\cdot2^{\iota/2}\ln N/\zeta}\\
&\le \frac{\sqrt{8c_0}}{j(z_i)-1+\lambda2^{\iota/2+1}\zeta/\ln N}\\
&\le \frac{\sqrt{8c_0}}{j(z_i)-1+\lambda2^{\iota+1}},
\$
where the last inequality uses $\zeta\ge2^{\iota/2}\ln N$. Summing over all $z_t\in\cZ^{\iota}$, we have
\#\label{eq:D_lambda_F_t_3}
\sum_{z_t\in\cZ^{\iota}} (D_{\lambda,\sigma,\cF}(Z_1^t))^2 &\le \sum_{j=1}^{N^{\iota}}\sum_{z_i\in\cZ_j^{\iota}} \frac{\sqrt{8c_0}}{j-1 + \lambda2^{\iota+1}} + \sum_{z_i\in\cZ_{N^{\iota}+1}^{\iota}} \frac{2}{N^{\iota}}\nonumber\\
&\le (\sqrt{8c_0}+2)\dim_E(\cF,2^{(-\iota-1)/2})\ln N^{\iota}.
\#

Eventually, combining \eqref{eq:D_lambda_F_t_1}, \eqref{eq:D_lambda_F_t_2} and \eqref{eq:D_lambda_F_t_3}, we obtain
\$
&\sum_{t=1}^T (D_{\lambda,\sigma,\cF}(Z_1^t))^2\\
&\quad= \sum_{\iota=0}^{\log(T/\lambda)} \sum_{z_t\in\cZ^{\iota}} (D_{\lambda,\sigma,\cF}(Z_1^t))^2\\
&\quad\le \sum_{\iota=0}^{\iota_0}(\sqrt{8c_0}+2)\dim_E(\cF,2^{(-\iota-1)/2})\ln N^{\iota} + \sum_{\iota=\iota_0+1}^{\log(T/\lambda)-1}4\dim_E(\cF,2^{(-\iota-1)/2})\ln N^{\iota} + 1\\
&\quad\le (\sqrt{8c_0}+3)\dim_E(\cF,\sqrt{\lambda/T})\log(T/\lambda)\ln T,
\$
where the last inequality uses the monotonicity of the eluder dimension. Note that if $\iota_0=-1$, let the sum from $0$ to $-1$ be $0$. Ultimately, we accomplish the proof due to the arbitrariness of $Z_1^T$.
\end{proof}

\section{Stable Bonus Funtion Space}\label{s:subsample}
In this section, we discuss how to stabilize the bonus function space to address the issue of covering numbers as stated in Section~\ref{sec:mdp_alg}. We first consider a case where $\cF$ is an RKHS so that the covering number can be directly controlled. Then, we adopt the sensitivity subsampling technique from \citet{wang2020reinforcement} for the general cases. 

\subsection{RKHS}
In this section, we consider that for each $h\in[H]$, the function class $\cF^h$ is embedded into some reproducing kernel Hilbert space (RKHS) $\cH^h$:
$$
\cF^h\subset\cH^h = \{\langle w(f), \phi(\cdot)\rangle: z \to \mathbb{R}\},
$$
where we use $z$ to denote state-action pair $(x,a)$. We assume that $\cH^h$ has a representation $\{\phi_j(z)\}_{j=1}^{\infty}$ in $2$-norm. For any $\epsilon>0$, define the $\epsilon$-scale sensitive dimension as follows:
$$
d(\epsilon) = \min\left\{ |\cS|: \sup_{z\in\cZ}\sum_{j\notin\cS}(\phi_j(z))^2 \le \epsilon\right\}.
$$
Then, we can write the bonus function as 
$$
b^h_t(z)= \sup_{f\in\cF_t^h} \frac{|\langle w(f)-w(\hat f) , \phi(z) \rangle|}{\sqrt{\lambda + (w(f)-w(\hat f))^{\top} \Lambda_t (w(f)-w(\hat f))}},
$$
where $\Lambda_t=\sum_{s=1}^{t-1}\phi(z_s)\phi(z_s)^{\top}/(\sigma_s^h)^2$.
Then, the log-covering number of the bonus space $\cB^h(\lambda)$ is bounded as follows.
\begin{lemma}\label{lm:rkhs_bonus_cover}
Suppose that $\|w\|\le1$, $\|\phi(x,a)\|\le1$ for all $(x,a)\in\cX\times\cA$, $\|\Lambda_t\|_{\op}\le\rho$ for all $t\in[T]$ and $\lambda>1$. Let $N(\gamma,\cB^h(\lambda))$ be the infinity $\gamma$-covering number of $\cB^h(\lambda)$. Then, we have
$$
\ln N(\gamma,\cB^h(\lambda)) = \tilde{\cO}\left((d((\gamma^2/46\rho)^2))^2\right).
$$
\end{lemma}
\begin{proof}
For any two functions $b_1,b_2\in\cB^h(\lambda)$ with parameters $w_1,w_2$ and $\Lambda_1,\Lambda_2$. Since $sup_w$ is a contraction map, we have
\$
\|b_1-b_2\|_{\infty} &\le \sup_{z,\|w\|\le1} \left| \frac{|\langle w-w_1,\phi(z)\rangle|}{\sqrt{\lambda + (w-w_1)^{\top}\Lambda_1(w-w_1)}} - \frac{|\langle w-w_2,\phi(z)\rangle|}{\sqrt{\lambda + (w-w_2)^{\top}\Lambda_2(w-w_2)}}\right|\\
&\le \sup_{\|\phi|\le 1,\|w\|\le1} \left\{\left| \frac{|\langle w-w_1,\phi\rangle|}{\sqrt{\lambda + (w-w_1)^{\top}\Lambda_1(w-w_1)}} - \frac{|\langle w-w_2,\phi\rangle|}{\sqrt{\lambda + (w-w_1)^{\top}\Lambda_1(w-w_1)}}\right|\right.\\
&\qquad \left.+\left| \frac{|\langle w-w_2,\phi\rangle|}{\sqrt{\lambda + (w-w_1)^{\top}\Lambda_1(w-w_1)}} - \frac{|\langle w-w_2,\phi\rangle|}{\sqrt{\lambda + (w-w_2)^{\top}\Lambda_2(w-w_2)}}\right|\right\}\\
&\le \sup_{\|\phi\|\le1} |\langle w_1-w_2,\phi\rangle|\\
&\qquad + \sup_{\|\phi|\le 1,\|w\|\le1} |\langle w-w_2,\phi\rangle| \cdot \left| \sqrt{\lambda + (w-w_2)^{\top}\Lambda_2(w-w_2)} - \sqrt{\lambda + (w-w_1)^{\top}\Lambda_1(w-w_1)}\right|\\
&\le \|w_1-w_2\| + 2 \sqrt{\underbrace{\sup_{\|w\|\le1}\left| (w-w_2)^{\top}\Lambda_2(w-w_2) - (w-w_1)^{\top}\Lambda_1(w-w_1)\right|}_{(a)}}.
\$
Then, we can bound (a) by
\$
&\sup_{\|w\|\le1} \left| (w-w_2)^{\top}\Lambda_2(w-w_2) - (w-w_1)^{\top}\Lambda_1(w-w_1)\right|\\
&\quad\le \sup_{\|w\|\le1} \left| (w-w_2)^{\top}(\Lambda_2-\Lambda_1)(w-w_2) + (w_1-w_2)^{\top}\Lambda_1(w_1-w_2)\right.\\
&\qquad\qquad\quad\left. + (w-w_1)^{\top}\Lambda_1(w_1-w_2) + (w_1-w_2)^{\top}\Lambda_1(w-w_1)\right|\\
&\le \|\Lambda_2-\Lambda_1\|_{\op} + \rho\|w_1-w_2\|^2 + 4\rho \|w_1-w_2\|,
\$
where the last inequality is obtained due to $\|\Lambda\|_{op} \le \rho$. Taking this result back, we get
\#\label{eq:b1-b2_infty}
\|b_1-b_2\|_{\infty} &\le \|w_1-w_2\| + 2\sqrt{\|\Lambda_2-\Lambda_1\|_{\op} + \rho\|w_1-w_2\|^2 + 4\rho \|w_1-w_2\|}.
\#
Let $\cC_w$ be the $\gamma^2/23\rho$ cover of $\{w\in\cH^h: \|w\|\le1\}$ and $\cC_{\Lambda}$ be the $\gamma^2/23$ cover of $\{\Lambda\in\cG: \|\Lambda\|_{\op}\le \rho\}$, where $\cG$ is defined in Lemma \ref{lm:covering_number_of_RKHS_ball}. By invoking Lemma \ref{lm:covering_number_of_RKHS_ball} and \ref{lm:covering_number_of_Hilbert_Inner_Product}, we have
\$
&\ln|\cC_w| \le d((\gamma^2/46\rho)^2)\ln(1+92\rho/\gamma^2)\\
&\ln|\cC_{\Lambda}| \le (d((\gamma^2/46\rho)^2))^2\ln(1+92\rho\sqrt{d((\gamma^2/46\rho)^2)}/\gamma^2).
\$
From \eqref{eq:b1-b2_infty}, for any $b_1\in\cB^h(\lambda)$ with $w(f)$ and $\Lambda=\sum_{s=1}^{t-1}\phi(z_s)\phi(z_s)^{\top}/(\sigma_s^h)^2$, let $\tilde \phi(z_s)=\phi(z_s)/\sigma_s^h$, so $\Lambda$ can be written as $\Lambda=\sum_{s=1}^{t-1}\tilde \phi(z_s) \tilde \phi(z_s)^{\top}$ with $\|\tilde \phi\|\le 1$, which indicates that $\Lambda\in\cG$. Hence, there always exist a $b_2$ parameterized by $\tilde w\in\cC_w$ and $A\in\cC_{\Lambda}$ such that $\|b_1 - b_2\|_{\infty}\le \gamma$.
Therefore, we finally obtain
\$
&\ln N(\gamma,\cB^h(\lambda)) \le \ln(|\cC_w|\cdot |\cC_{\Lambda}|)\\
&\quad\le d((\gamma^2/46\rho)^2)\ln(1+92\rho/\gamma^2) + (d((\gamma^2/46\rho)^2))^2\ln(1+92\rho\sqrt{d((\gamma^2/46\rho)^2)}/\gamma^2).
\$
\end{proof}

\begin{lemma}[Covering number of Covariance Matrix]\label{lm:covering_number_of_RKHS_ball} 
Consider the space 
$$
\cG=\left\{\sum_{s=1}^n \phi(z_s)\phi(z_s)^{\top}: \phi\in\cH^h,\|\phi\|\le1, n\in\mathbb N\right\}.
$$
For any $\epsilon>0$, let $N(\gamma,\cG,R)$ denote the $\gamma$-covering number of the ball $\cB_{\cG}(R)=\{\Lambda\in\cG:\|\Lambda\|_{\op}\le R\}$ with respect to the operator norm. Then, its log-covering number with respect to operator norm is bounded by
$$
N(\gamma,\cG,R)\le (d((\gamma/2R)^2))^2\ln(1+4R\sqrt{d((\gamma/2R)^2)}/\gamma).
$$
\end{lemma}
\begin{proof}
First of all, we consider the log-covering number of a RKHS ball. For any $\epsilon>0$, let $N(\gamma,\cH^h,R)$ denote the log-$\gamma$-covering number of the ball $\cB_{\cH^h}(R)=\{\phi\in\cH^h:\|\phi\|\le R\}$ with respect to the norm $\|\cdot\|_{\cH}$ on the Hilbert space. We claim that its log-covering number is bounded by
$$N_{\phi}(\gamma,\cH^h,R) \le
d((\gamma/2R)^2)\ln(1+4R/\gamma).
$$
Recall that $\{\phi_j\}_{j=1}^{\infty}$ are representations of $\cH^h$. There exists an index set $\cS$ with the smallest cardinality such that 
$$
\sup_{z\in\cZ}\sum_{j\notin\cS} (\phi_j)^2 \le (\frac{\gamma}{2R})^2 := \gamma'.
$$
Hence, we get $|\cS|=d(\gamma')$. Then, any $\phi\in\cH^h$ can be written as $\phi=\sum_{j=1}^{\infty}\alpha_j\phi_je_j$, where $e_j$ is the infinite-dimensional vector with only index $j$ equaling $1$, we define $\Pi_{\cS}:\cH^h\rightarrow\cH^h$ as the projection map onto the subspace spanned by $\{\phi_je_j\}_{j\in\cS}$:
$$
\Pi_{\cS}(\phi) = \sum_{j\in\cS} \alpha_j\phi_je_j.
$$
For any $\phi\in\cB_{\cH^h}(R)$, notice that $\|\alpha\|\le R$, so we have
\$
\|\phi - \Pi_{\cS}(\phi)\| \le \sqrt{\sum_{j\notin\cS}\alpha_j^2\phi_j^2}
\le R\sqrt{\sum_{j\notin\cS} \phi_j^2}
\le R\sqrt{\gamma'} = \gamma/2.
\$
Then, Let $\cC(\gamma/2,d(\gamma'),R)$ be the $\gamma/2$-cover of $\{\alpha\in\rR^{d(\gamma')}: \|\alpha\|\le R\}$ with respect to the Euclidean norm. Thus, there exists $\tilde{\alpha}\in\cC(\gamma/2,d(\gamma'),R)$ such that $\sum_{j\in\cS}(\alpha_j-\tilde{\alpha}_j)^2\le (\gamma/2)^2$. It follows that for any $z\in\cZ$,
\$
\left\|\Pi_{\cS}(\phi) - \sum_{j\in\cS}\tilde{\alpha}_je_j\phi_j \right\| &= \left\|\sum_{j\in\cS}(\alpha_j-\tilde{\alpha}_j)e_j\phi_j \right\|\\
&\le \sqrt{\sum_{j\in\cS}(\alpha_j-\tilde{\alpha}_j)^2\cdot\phi_j^2}\\
&\le \gamma/2.
\$
Therefore, the following result is obtained:
\$
\left\| \phi - \sum_{j\in\cS}\tilde{\alpha}_j\phi_je_j\right\| \le \|\phi - \Pi_{\cS}(\phi)\| + \left\|\Pi_{\cS}(\phi) - \sum_{j\in\cS}\tilde{\alpha}_j\phi_je_j \right\|\le \gamma,
\$
which indicates that 
$$
N_{\phi}(\gamma,\cH^h,R)\le |\cC(\gamma/2,d(\gamma'),R)| \le \ln\left((1+4R/\gamma)^{d(\gamma')}\right) = d(\gamma')\ln(1+4R/\gamma),
$$
where the second inequality uses Lemma 5.2 in \citet{vershynin2010introduction}.

Further, let $\cC(\gamma/2,d^2(\gamma'),R)$ be the $\gamma/2$-cover of $\{A\in\rR^{d(\gamma')\times d(\gamma')}:\|A\|_{\op}\le R\}$ with respect to the Frobenius norm. For a covariance matrix $\Lambda\in\cG$, there exists a matrix $A\in\cC(\gamma/2,d^2(\gamma'),R)$ with $\|A\|_{\op}\le t$ such that
$$
\|\Lambda-A\|_{\op}\le \|\Lambda-A\|_F \le \gamma,
$$
where the second inequality is obtained from the claim in the vector case. For any matrix $A\in\rR^{d(\gamma')\times d(\gamma')}$ with operator norm bounded by $R$, its Frobenius norm is bounded by $R\sqrt{d(\gamma')}$. Therefore, we can bound the log-covering number for the matrix space as
$$
N(\gamma,\cG,R)\le d^2(\gamma')\ln(1+4R\sqrt{d(\gamma')}/\gamma).
$$
\end{proof}

\begin{lemma}[Covering Number of Hilbert Inner Product]\label{lm:covering_number_of_Hilbert_Inner_Product}
For any $\epsilon>0$, let $N_w(\gamma,\cH^h,R)$ be the log-$\gamma$-covering number of the Ball $\{w\in\cH^h:\|w\|\le R\}$. Then, supposing that $\|\phi\|\le1$, we have
$$
N_w(\gamma,\cH^h,R) \le d((\gamma/2R)^2)\ln(1+4R/\gamma)
$$
\end{lemma}
\begin{proof}
Consider the representations $\{\phi_j\}_{j=1}^{\infty}$ of $\cH$ and the core index set $\cS$ such that $\sup_{z\in\cZ}\sum_{j\notin\cS} (\phi_j)^2 \le (\frac{\gamma}{2R})^2 = \gamma'$. For any $f\in\cF^h$, we can write it as $f(z)=\sum_{j=1}^{\infty}w_j\phi_j(z)$. Define $\Pi_{\cS}:\cH^h\rightarrow\cH^h$ as the projection map onto the subspace spanned by $\{\phi_je_j\}_{j\in\cS}$:
$$
\Pi_{\cS}(f) = \sum_{j\in\cS} w_j\phi_j.
$$
For any $f\in\cB_{\cH^h}(R):=\{f\in\cH^h: \|f\|_{\cH^h}\le R\}$, we have $\|w\|\le R$, which implies that
\$
\|f - \Pi_{\cS}(f)\|_{\cH^h} = \left\|\sum_{j\notin\cS}w_j\phi_j\right\|_{\cH}
\le \sqrt{\sum_{j\notin\cS}w_j^2}\cdot\sup_{z\in\cZ}\sqrt{\sum_{j\notin\cS} (\phi_j(z))^2} \le R\sqrt{\gamma'} = \gamma/2.
\$
Now, use $\cC(\gamma/2,d(\gamma'),R)$ to denote the $\gamma/2$-cover of $\{w\in\rR^{d(\gamma')}: \|w\|\le R\}$ with respect to the Euclidean norm. There exists $\tilde{w}\in\cC(\gamma/2,d(\gamma'),R)$ such that $\sum_{j\in\cS}(w_j-\tilde{w_j})^2\le (\gamma/2)^2$. Therefore, for any $z\in\cZ$,
\$
\left\|\Pi_{\cS}(f) - \sum_{j\in\cS}\tilde{w_j}\phi_j \right\|_{\cH^h} &= \left\|\sum_{j\in\cS}(w_j-\tilde{w_j})\phi_j \right\|_{\cH^h}\\
&\le \sqrt{\sum_{j\in\cS}(w_j-\tilde{w_j})^2}\cdot
\sup_{z\in\cZ}\sqrt{\sum_{j\notin\cS} (\phi_j(z))^2}\\
&\le \gamma/2.
\$
Therefore, the following result is obtained:
\$
\left\| f - \sum_{j\in\cS}\tilde{w_j}\phi_j\right\|_{\cH^h} \le \|f - \Pi_{\cS}(f)\|_{\cH^h} + \left\|\Pi_{\cS}(f) - \sum_{j\in\cS}\tilde{w_j}\phi_j \right\|_{\cH^h} \le \gamma.
\$
Finally, for any parameter $w=[w_1,w_2,\ldots]$ inducing the function $f=\sum_{j=1}^{\infty}w_j\phi_j$, there exists a vector $\tilde w = [\tilde w_j]_{j\in\cS}$ such that
$$
\|w-\tilde w\| = \left\| f - \sum_{j\in\cS}\tilde{w_j}\phi_j\right\|_{\cH^h} \le \gamma.
$$
Then, by invoking Lemma 5.2 in \citet{vershynin2010introduction}, we attain the result:
$$
N_w(\gamma,\cH^h,R)\le |\cC(\gamma/2,d(\gamma'),R)| \le \ln(1+4R/\gamma)^{d(\gamma')} = d(\gamma')\ln(1+4R/\gamma).
$$
\end{proof}

\subsection{Stabilization with Subsampling}
The main idea of the sub-sampling procedure in \citet{wang2020reinforcement} is to compute the bonus by a sub-sampled dataset. This subset enjoys a log-covering number that is upper bounded by the eluder dimension because the number of distinct elements of the subsampled dataset is much smaller than the original one. Meanwhile, the uncertainty estimations are roughly preserved using the sub-sampled dataset. To achieve the goal, we measure the importance of the data points in the dataset and only maintain those critical for the bonus construction (by importance sampling). 

However, the extension is not straightforward because the uncertainty-weighted regression brings distinct challenges. Controlling the uncertainty level in the face of weights also plays a central role in our analysis. Before continuing, we additionally assume that the covering number of state-action pairs $\cX\times\cA$ is bounded, which is also required by \citet{wang2020reinforcement, kong2021online}.
\begin{assumption}\label{as:cover_state_action_pair}
There exists a $\gamma$-cover $\cC(\cX\times\cA,\gamma)$ with bounded size $N(\cX\times\cA,\gamma)$ such that for any $(x,a)\in\cX\times\cA$, there exists $(x',a')\in\cX\times\cA$ satisfying that $\sup_{f\in\cF}|f(x,a)-f(x',a')|\le\epsilon$.
\end{assumption}

For simplicity, we denote the space $\cX\times\cA$ by $\cZ$. In the sub-sampling algorithm, a core set is generated according to each element's sensitivity, the error on one sample compared to the whole-sample error. We first define 
$$\SEN_{\lambda, \sigma, \cF, \cZ}(z) = \min\left\{1, \sup_{f,g \in \cF} \frac{(f(z) - g(z))^2/\sigma_z^2}{\lambda + \norm{f-g}^2_{\cZ, \sigma}}\right\}.$$
This sensitivity resembles the quantity considered in \eqref{def:eluder_dim} except that the former is compared to the whole sample set, and the latter is compared to the history samples. Thus, we can bound the sum of sensitivities by invoking Lemma \ref{lm:Relationship_between_Eluder_Dimensions}.

\begin{lemma}\label{lm:sum_sensitvity}
For a given set of state-action pairs $Z_1^t$ and the corresponding sequence of weights, we have
$$
\sum_{z \in Z_1^t} \SEN_{\lambda, \sigma, \cF_t^h, Z_1^t}(z) \leq \sum_{s=1}^{t} D_{\lambda,\sigma,\cF}(Z_1^s)^2 = \tilde{\cO}(\dim_E(\cF, \lambda/T)).
$$
\end{lemma}
This lemma plays the same role as Lemma 1 in \citet{wang2020reinforcement}. Let $\hat{\cZ}$ denote the core set generated during Algorithm 3 in \citet{wang2020reinforcement}. We construct the confidence set based on $\hat{\cZ}$:
$$
\hat{\cF}_t^h = \left\{f^h\in\cF_{t-1}^h: \lambda + \|f^h-\bar{f}\|_{\hat{\cZ},\sigma}^2 \le 3\beta^2+2 \right\},
$$
and the bonus function
$\hat b_t^h(\cdot)\in\hat{\cB}^h(\lambda)$:
$$
\hat b_t^h(z) = \sup_{f^h\in\hat{\cF}_t^h} \frac{|f^h(z)-\hat{f}_t^h(z)|}{\sqrt{\lambda + \|f^h-\bar{f}\|_{\hat{\cZ},\sigma}^2}},
$$
where $\bar{f}\in\cF_{\gamma}^h$ is close to $\hat{f}_t^h$ in the sense that $\|\bar{f}-\hat{f}_t^h\|_{\infty}\le\gamma$. Combining Lemma \ref{lm:sum_sensitvity} and Lemma 2 to 4 in \citet{wang2020reinforcement}, we obtain the following lemma.
\begin{lemma}
If $|\cZ|\le TH$, the following results hold:
\begin{enumerate}
\item With probability at least $1-\delta/(16TH)$, for all $h\in[H]$, we have $$b_t^h(z)\le \sqrt{2}\hat b_t^h(z).$$
\item If $T$ is sufficiently large,
$$
\ln|\hat{\cB}^h(\lambda)| = \tilde{\cO}\left(\dim_E(\cF^h,\delta/(TH)^3)\cdot \ln(N(\cF^h,\delta/(TH)^2)/\delta)\cdot \log(N(\cX\times\cA,\delta/(TH)))\right).
$$
\end{enumerate}
\end{lemma}
Then, by replacing the bonus $b_t^h(z)$ with $\hat b_t^h(z)$ in Algorithm \ref{alg:ucb_nonlinear_mdp}, we can bound the regret in Theorem \ref{th:mdp} by
$$
\reg(T) = \Tilde{\cO}\left(\sqrt{TH\iota} + \zeta\dim_E(\cF,\lambda/T)\right),
$$
where
\$
\iota &= \ln|\hat{\cB}^h(\lambda)|\dim_E(\cF,\lambda/T)\\
&= \dim_E(\cF^h,\delta/(TH)^3)^2\cdot \ln(N(\cF^h,\delta/(TH)^2)/\delta)\cdot \log(N(\cX\times\cA,\delta/(TH))).
\$

\section{Unknown Corruption Level}\label{s:Unknown_Corruption_Level}
\begin{proof}[Proof of Theorem \ref{th:mdp_unknown}]
We first consider the case when $\zeta\le \bar\zeta$. Since only the predetermined parameters $\alpha$ and $\beta$ are modified by replacing $\zeta$ with $\bar\zeta$ and $\zeta$ is upper bounded by $\bar\zeta$, it is easy to obtain the regret bound $\reg(T)=\Tilde{\cO}(\sqrt{TH\ln(N_T(\gamma))\dim_E(\cF,\lambda/T)} + \bar\zeta\dim_E(\cF,\lambda/T))$ by following the analysis of Theorem \ref{th:mdp}. Additionally, Lemma~\ref{lm:Relationship_between_Eluder_Dimensions} can also be proved by discussing the value of $\bar\zeta$ in the first step.

Then, for the case when $\zeta>\bar\zeta$, we simply take the trivial bound $T$.
\end{proof}

\section{Technical Lemmas}
\begin{lemma}[Value-decomposition Lemma \citep{jiang@2017}]\label{lm:bellman_error}
Consider any candidate value function $f=\{f^h(x^h,a^h):\cX\times\cA\rightarrow\rR\}$, with $f^{H+1}(\cdot)=0$. Let $\pi_f$ be its greedy policy. We have
$$
f^1(x^1)-V_{\pi_f}^1(x^1)=\E_{\pi_f}\left[\sum_{h=1}^H \cE^h(f,x^h,a^h)\,\Big|\,x^1\right],
$$
where $\E_{\pi_f}[\cdot|x^h]$ means taking expectation over the trajectory $\{(x^{h'},a^{h'})\}_{h'=h}^H$ induced by policy $\pi_f$ starting from the state $x^h$.
\end{lemma}

\begin{lemma}\label{lm:epsilon}
Let $\{\epsilon_s\}$ be a sequence of zero-mean conditional $\sigma$-sub-Gaussian random variables: $\ln\E[e^{\lambda\epsilon_i}|\cS_{i-1}]\le\lambda^2\sigma^2/2$, where $\cS_{i-1}$ represents the history data. We have for $t\ge1$, with probability at least $1-\delta$,
$$
\sum_{s=1}^t \epsilon_i^2\le 2t\sigma^2 + 3\sigma^2\ln(1/\delta).
$$
\end{lemma}
\begin{proof}
By invoking the logarithmic moment generating function estimate in Theorem 2.29 from \citet{TZ23-lt}, we know that for $\lambda\ge0$,
\#\label{eq:conditional_gen}
\ln\E\left[\exp\big(\lambda\epsilon_i^2\big) \,|\, \cS_{i-1} \right] \le \lambda\sigma^2 + \frac{(\lambda\sigma^2)^2}{1-2\lambda\sigma^2}.
\#
Then, by using iterated expectations due to the tower property of conditional expectation, we get
\$
\E\left[\exp\Big(\lambda\sum_{i=1}^t\epsilon_i^2\Big)\right] &= \E\left\{\E\left[\exp\Big(\lambda\sum_{i=1}^{t-1}\epsilon_i^2 + \epsilon_t^2\Big) \,\Big|\, \cS_{t-1} \right]\right\}\\
&= \E\left\{\exp\Big(\lambda\sum_{i=1}^{t-1}\epsilon_i^2\Big)\cdot \E\left[\exp\big(\epsilon_t^2\big) \,\Big|\, \cS_{t-1} \right]\right\}\\
&\le \exp\big(\lambda\sigma^2 + \frac{(\lambda\sigma^2)^2}{1-2\lambda\sigma^2}\big)\cdot\E\left\{\exp\Big(\lambda\sum_{i=1}^{t-1}\epsilon_i^2\Big)\right\}\\
\ldots &\le \exp\big(\lambda t\sigma^2 + \frac{(\lambda t\sigma^2)^2}{1-2\lambda\sigma^2}\big),
\$
where the first inequality uses \eqref{eq:conditional_gen}.
Now, we can apply the second inequality of Lemma 2.9 from \citep{TZ23-lt} with $\mu=t\sigma^2, \alpha=2t\sigma^4, \beta=2\sigma^2$ and $\epsilon=2\sigma^2\sqrt{ut}$ to obtain
\#\label{eq:lem2.9}
\inf_{\lambda\ge0}\left\{ -\lambda\big(t\sigma^2 + 2\sqrt{ut\sigma^4} + 2u\sigma^2\big) + \ln\E\left[\exp\Big(\lambda\sum_{i=1}^t\epsilon_i^2\Big)\right]\right\} \le -u.
\#
Thus, it follows that
\$
&\Pb\left( \sum_{s=1}^t \epsilon_i^2 \le t\sigma^2 + 2\sqrt{ut\sigma^4} + 2u\sigma^2\right)\\
&\quad\le \inf_{\lambda\ge0}\frac{\E\left[\exp\Big(\lambda\sum_{i=1}^t\epsilon_i^2\Big)\right]}{\exp\big(\lambda (t\sigma^2 + 2\sqrt{ut\sigma^4} + 2u\sigma^2)\big)}\\
&\quad= \inf_{\lambda\ge0} \exp\left( -\lambda\big(t\sigma^2 + 2\sqrt{ut\sigma^4} + 2u\sigma^2\big) + \ln\E\left[\exp\Big(\lambda\sum_{i=1}^t\epsilon_i^2\Big)\right]\right)\\
&\quad\le e^{-u},
\$
where the first inequality applies Markov’s Inequality, and the second inequality uses \eqref{eq:lem2.9} and the monotonicity of the exponential function.

Taking $u=\ln(1/\delta)$ for $\delta>0$, we obtain that with probability at least $1-\delta$
\$
\sum_{s=1}^t \epsilon_i^2 &\le t\sigma^2 + 2\sqrt{t\ln(1/\delta)\sigma^4} + 2\ln(1/\delta)\sigma^2\\
&\le 2t\sigma^2 + 3\sigma^2\ln(1/\delta),
\$
where the second inequality is deduced since $2\sqrt{t\ln(1/\delta)\sigma^4}\le t\sigma^2 + \ln(1/\delta)\sigma^2$.
\end{proof}

\begin{lemma}\label{lm:th13_9}
Consider a sequence of random variables $\{Z_t\}_{t\in\mathbb N}$ adapted to the filtration $\{\cS_t\}_{t\in\mathbb N}$ and a function $f\in\cF$. For any $\lambda>0$, with probability at least $1-\delta$, for all $t\ge1$, we have
$$
-\sum_{s=1}^tf(Z_s)-\frac{1}{\lambda}\sum_{s=1}^t\ln\E[e^{-\lambda f(Z_s)}|\cS_{s-1}]\le \frac{1}{\lambda\delta}
$$
\end{lemma}
\begin{proof}
The proof of this lemma is presented in Lemma 4 of \citet{russo2013eluder}.
\end{proof}

Then, we demonstrate the convergence of training error for both ERM and approximate ERM solutions in the following lemma.
\begin{lemma}\label{lm:empirical_diff_mdp}
Consider a function space $\cF:\cZ\rightarrow\rR$
and filtered sequence $\{z_t,\epsilon_t\}$ in $\cX\times\rR$ so that $\epsilon_t$ is conditional zero-mean $\eta$-sub-Gaussian noise. For $f_*(\cdot):\cZ\rightarrow\rR$, suppose that $y_t=f_*(z_t)+\epsilon_t$ and there exists a function $f_b\in\cF$ such that for any $t\in[T]$,
$
\sum_{s=1}^t|f_*(z_s)-f_b(z_s)| := \sum_{s=1}^t\zeta_s \le \zeta$. If $\hat{f}_t$ is an (approximate) ERM solution for some $\epsilon'\ge0$:
$$
\left(\sum_{s=1}^t(\hat{f}_t(z_s)-y_s)^2/\sigma_s^2\right)^{1/2} \le \min_{f\in\cF_{t-1}} \left(\sum_{s=1}^t(f(z_s)-y_s)^2/\sigma_s^2\right)^{1/2} + \sqrt{t}\epsilon',
$$
with probability at least $1-\delta$, we have for all $t\in[T]$:
\$
\sum_{s=1}^t(\hat{f}_t(z_s)-f_b(z_s))^2/\sigma_s^2 \le& 10\eta^2\ln(2N(\gamma,\cF)/\delta) + 5\sum_{s=1}^t|\hat{f_t}(z_s)-f_b(z_s)|\zeta_s/\sigma_s^2\\
&\quad + 10(\gamma+\epsilon')((\gamma+\epsilon')t + \sqrt{tC_1(t,\zeta)}),
\$
where $C_1(t,\zeta)=2(\zeta^2 + 2t\eta^2 + 3\eta^2\ln(2/\delta))$.
\end{lemma}
\begin{proof}
For $f\in\cF$, define
$$
\phi(f,z_t)=-a\left[ (f(z_t)-y_t)^2-(f_b(z_t)-y_t)^2\right]/\sigma_t^2,
$$
where $a=\eta^{-2}/4$. Let $\cF_{\gamma}$ be an $\gamma$-cover of $\cF$ in the sup-norm. Denote the cardinality of $\cF_{\gamma}$ by $N=N(\gamma,\cF)$.

Since $\epsilon_t$ is conditional $\eta$-sub-Gaussian and $\phi(f,z_t)$ can be written as
\$
\phi(f,z_t) =2a(f(z_t)-f_b(z_t))/\sigma_t^2\cdot\epsilon_t -a(f(z_t)-f_b(z_t))^2/\sigma_t^2+2a(f(z_t)-f_b(z_t))\zeta_t/\sigma_t^2,
\$
the $\phi(f,z_t)$
is conditional $2a(f(z_t)-f_b(z_t))\eta/\sigma_t^2$-sub-Gaussian with mean 
\$
\mu=-a(f(z_t)-f_b(z_t))^2/\sigma_t^2+2a(f(z_t)-f_b(z_t))\zeta_t/\sigma_t^2,
\$
where $a=\eta^{-2}/4$. We know that if a variable $X$ is $\sigma$-sub-Gaussian with mean $\mu$ conditional on $\cS$, the property of sub-Gaussianity implies that
\$
\ln\mathbb E\big[\exp(s(X-\mu))\,\big|\,\cS\big] \le \frac{\sigma^2s^2}{2}.
\$
By taking $s=1$ in the inequality above, we get
\$
\ln\mathbb E_{y_t}[\exp(\phi(f,z_t) - \mu)\,|\,z_t,\mathcal S_{t-1}] \le \frac{4a^2(f(z_t)-f_b(z_t))^2\eta^2}{2\sigma_t^4} = \frac{(f(z_t)-f_b(z_t))^2}{8\eta^2\sigma_t^4}.
\$

It follows that
\$
\ln\mathbb E_{y_t}[\exp(\phi(f,z_t))\,|\,z_t,\mathcal S_{t-1}] &\le \frac{(f(z_t)-f_b(z_t))^2}{8\eta^2\sigma_t^4} -\frac{(f(z_t)-f_b(z_t))^2}{4\eta^2\sigma_t^2}+\frac{(f(z_t)-f_b(z_t))\zeta_t}{2\eta^2\sigma_t^2}\\
&\le \frac{(f(z_t)-f_b(z_t))^2}{8\eta^2\sigma_t^2} -\frac{(f(z_t)-f_b(z_t))^2}{4\eta^2\sigma_t^2}+\frac{(f(z_t)-f_b(z_t))\zeta_t}{2\eta^2\sigma_t^2}\\
&\le -\frac{(f(z_t)-f_b(z_t))^2}{8\eta^2\sigma_t^2}+\frac{(f(z_t)-f_b(z_t))\zeta_t}{2\eta^2\sigma_t^2},
\$
where the second inequality uses $\sigma_t^2\ge1$. 

Invoking Lemma \ref{lm:th13_9} with $\lambda=1$, we have for all $f\in\cF_{\gamma}$ and $t\in[T]$, with probability at least $1-\delta/2$,
\#\label{eq:sum_phi_f_zs}
\sum_{s=1}^t\phi(f,z_s) \le -\sum_{s=1}^t\frac{(f(z_s)-f_b(z_s))^2}{8\eta^2\sigma_s^2}+\sum_{s=1}^t\frac{(f(z_s)-f_b(z_s))\zeta_s}{2\eta^2\sigma_s^2}+\ln(2N/\delta).
\#
Additionally, for all $t\in[T]$, we have with probability at least $1-\delta/2$,
\#\label{eq:sum_fbZs_Ys}
\sum_{s=1}^t(f_b(z_s)-y_s)^2 &= \sum_{s=1}^t (f_b(z_s) - f_*(z_s) + f_*(z_s) - y_s)^2\nonumber\\
&\le 2\sum_{s=1}^{t-1} \left( (f_b(z_s)-f_*(z_s))^2 + (f_*(z_s)-y_s)^2 \right)\nonumber\\
&\le 2\left(\sum_{s=1}^{t-1}\zeta_s^2 + \sum_{s=1}^{t-1}\epsilon_s^2\right)\nonumber\\
&\le 2\left(\zeta^2 + 2t\eta^2 + 3\eta^2\ln(2/\delta)\right):=C_1(t,\zeta),
\#
where the first inequality is obtained since Cauchy-Schwarz inequality, the second inequality uses $|f_b(z_s) - f_*(z_s)|\le \zeta_s$, and the last inequality comes from Lemma \ref{lm:epsilon}. To simplify notations, we use $C_1(t,\zeta)$ to denote $2(\zeta^2 + 2t\eta^2 + 3\eta^2\ln(2/\delta))$.
Now, given $\hat{f_t}$, there exists $f\in\cF_{\gamma}$ so that $\|\hat{f_t}-f\|_{\infty}\le\gamma$. With probability at least $1-\delta/2$,
\#\label{eq:sum_f_Rs_fb_Rs}
&\sum_{s=1}^t\left[ (f(z_s)-y_s)^2-(f_b(z_s)-y_s)^2\right]/\sigma_s^2\nonumber\\
&\quad\le \left(\sqrt{\sum_{s=1}^t(\hat{f}_t(z_s)-y_s)^2/\sigma_s^2} + \sqrt{t}\gamma\right)^2 - \sum_{s=1}^t(f_b(z_s)-y_s)^2/\sigma_s^2\nonumber\\
&\quad\le \left(\sqrt{\sum_{s=1}^t(f_b(z_s)-y_s)^2/\sigma_s^2} + \sqrt{t}(\gamma+\epsilon')\right)^2 - \sum_{s=1}^t(f_b(z_s)-y_s)^2/\sigma_s^2\nonumber\\
&\quad\le (\gamma+\epsilon')^2t + 2(\gamma+\epsilon')\sqrt{tC_1(t,\zeta)},
\#
where the first inequality uses $|f(z_s)-\hat{f_t}(z_s)|\le\gamma$ and triangle inequality for all $s$, the second inequality is because of the condition that $\hat{f_t}$ is an approximate ERM solution, and the last inequality utilizes $\sigma_s\ge1$ and \eqref{eq:sum_fbZs_Ys}. Finally, with probability at least $1-\delta$, we get
\$  
&\left(\sum_{s=1}^t(\hat{f_t}(z_s)-f_b(z_s))^2/\sigma_s^2\right)^{1/2}\\
&\quad\le \sqrt{\gamma^2t}+\left(\sum_{s=1}^t(f(z_s)-f_b(z_s))^2/\sigma_s^2\right)^{1/2}\\
&\quad\le \sqrt{\gamma^2t}+\left(4\sum_{s=1}^t(f(z_s)-f_b(z_s))\zeta_s/\sigma_s^2+8\eta^2\ln(2N/\delta)-8\eta^2\sum_{s=1}^t\phi(f,z_s)\right)^{1/2}\\
&\quad\le \sqrt{\gamma^2t}+\left(4\sum_{s=1}^t|\hat{f_t}(z_s)-f_b(z_s)|\zeta_s/\sigma_s^2 + 4\gamma\zeta +8\eta^2\ln(2N/\delta)+2(\gamma+\epsilon')^2t\right.\\
&\qquad \left.+ 4(\gamma+\epsilon')\sqrt{t}C_1'(t,\zeta)\right)^{1/2}\\
&\quad\le \left(10\eta^2\ln(2N/\delta)+5\sum_{s=1}^t|\hat{f_t}(z_s)-f_b(z_s)|\zeta_s/\sigma_s^2 + 5\gamma\zeta+8(\gamma+\epsilon')^2t+5(\gamma+\epsilon')\sqrt{tC_1(t,\zeta)}\right)^{1/2},
\$
where the second inequality is deduced from \eqref{eq:sum_phi_f_zs}, the third inequality is by the definition of $\sigma_s$ in \eqref{eq:weight_mdp} and the last inequality uses Cauchy-Schwarz inequality.
\end{proof}

\end{document}